\documentclass[10pt,twocolumn]{article}

\usepackage{amsmath,mathrsfs,amsthm}

\usepackage{multirow}
\usepackage{enumitem}
\usepackage{amssymb}
\usepackage{nicefrac,booktabs}
\usepackage{dsfont}
\usepackage[mathscr]{euscript}
\usepackage{nth}
\usepackage[dvipsnames,svgnames]{xcolor}
\usepackage{graphicx}
\usepackage{tikz}
\usetikzlibrary{arrows.meta, positioning}

\usepackage{caption}
\usepackage{subcaption}

\usepackage{stmaryrd}
\usepackage{amsmath,amsfonts,bm}

\makeatletter
\newcommand{\manuallabel}[2]{\def\@currentlabel{#2}\label{#1}}
\makeatother

\theoremstyle{plain}
\newtheorem{theorem}{Theorem}[section]

\theoremstyle{definition}
\newtheorem{definition}[theorem]{Definition}

\theoremstyle{definition}
\newtheorem*{example*}{Example}
\newtheorem*{condition*}{Condition}

\theoremstyle{remark}

\newcommand{\mf}[1]{\mathbf{#1}}

\newcommand{\rmm}[1]{\mathrm{#1}}
\newcommand{\R}{\mathbb{R}}

\newcommand{\A}{\mathbb{A}}

\newcommand{\E}{\mathbb{E}} 

\newcommand{\cE}{\mathcal{E}}
\newcommand{\cP}{\mathcal{P}}\newcommand{\cT}{\mathcal{T}}

\newcommand{\cU}{\mathcal{U}}

\newcommand{\cG}{\mathcal{G}}
\newcommand{\cR}{\mathcal{R}}
\newcommand{\cV}{\mathcal{V}}
\renewcommand{\cP}{\mathcal{P}}

\newcommand{\fR}{\mathfrak{R}}

\newcommand{\fE}{\mathfrak{E}}
\newcommand{\fT}{\mathfrak{T}}
\newcommand{\fK}{\mathfrak{K}}

\DeclareMathAlphabet{\mathpzc}{OT1}{pzc}{m}{it}

\newcommand{\BRac}[1]{\Big(#1\Big)}

\newcommand{\set}[1]{\left\{#1\right\}}

\newcommand{\norm}[1]{\left\vert#1\right\vert}

\newcommand{\mgg}{\mathrm{g}}
\newcommand{\Enc}[3]{\mf{h}^{(#3)}_{#1\vert #2}}
\newcommand{\EncR}[1]{\mf{R}_{\vert #1}}
\newcommand{\mc}[1]{\mathcal{#1}}
\newcommand{\Motif}{\rmm{Motif}}
\newcommand{\Ultra}{\rmm{Ultra^+}}
\newcommand{\triad}[4]{
  \begin{scope}[shift={(#1,#2)}]
    \node (a) at (0,0) {};
    \node (b) at (1,0) {};
    \node (c) at (0.5,0.866) {};  
    #3  
    \node at (0.5,-0.5) {#4}; 
  \end{scope}
}

\newcommand{\triadT}[4]{
  \begin{scope}[shift={(#1,#2)}]
    \node (a) at (0,0) {};
    \node (b) at (1,0) {};
    \node (c) at (1,0.866) {};
    \node (d) at (0,0.866) {};
    #3  
    \node at (0.5,-0.5) {#4}; 
  \end{scope}
}
\newcommand{\triadC}[4]{
  \begin{scope}[shift={(#1,#2)}]
    \node (a) at (0,0) {};
    \node (b) at (1,0) {};
    \node (c) at (1,0.866) {};
    \node (d) at (0,0.866) {};
    \node (e) at (1,1.566) {}; 
    #3  
    \node at (0.5,-0.5) {#4}; 
  \end{scope}
}

\newcommand{\rMany}[4]{\rmm{\textcolor{ForestGreen!80!black}{#1}#2_{\textcolor{ForestGreen!80!black}{#3}:#4} }}
\newcommand{\Ultrap}[1]{\Ultra[\cV_{#1}^{+}]}
\newcommand{\Ultram}[1]{\Ultra[\cV_{#1}^{-}]}
\newcommand{\Ultrav}[1]{\Ultra[\cV_{#1}]}
\newcommand{\Ultrau}[1]{\Ultra[\cU_{#1}]}

\usepackage{listings}
\lstdefinelanguage{SPARQL}{
  morekeywords={SELECT,CONSTRUCT,ASK,DESCRIBE,WHERE,FROM,NAMED,PREFIX,BASE,OPTIONAL,FILTER,GRAPH,LIMIT,OFFSET,ORDER,BY,ASC,DESC,DISTINCT,REDUCED},
  sensitive=true,
  morecomment=[l]{\#},
  morestring=[b]",
  morestring=[b]',
}

\usepackage[utf8]{inputenc}
\usepackage[T1]{fontenc}
\usepackage[margin=1in]{geometry}
\usepackage{url}
\usepackage{microtype}
\usepackage{algorithm}
\usepackage{algorithmic}
\usepackage{longtable}
\usepackage{newfloat}
\usepackage{listings}
\usepackage{wrapfig}
\usepackage{colortbl}
\usepackage[round]{natbib}
\usepackage{hyperref}
\usepackage[capitalize,noabbrev]{cleveref}
\usepackage{footnote}

\hypersetup{
	pdftitle={Graphlets as Building Blocks for Structural Vocabulary in Knowledge Graph Foundation Models},
	pdfauthor={Kossi Amouzouvi},
	pdfsubject={Knowledge graph foundation models},
	pdfkeywords={Knowledge Graph, Graph Foundation Model, Knowledge Graph Representation Learning, Graphlets},
	colorlinks=true,
	linkcolor=blue,
	citecolor=blue,
	urlcolor=blue
}

\title{Graphlets as Building Blocks for Structural Vocabulary in Knowledge Graph Foundation Models}
\author{%
	\parbox{0.96\textwidth}{\centering
		Kossi Amouzouvi\textsuperscript{*1,2}, Robert Wardenga\textsuperscript{*3}, Jens Lehmann\textsuperscript{**4}, Sahar Vahdati\textsuperscript{1,5}\\[0.25em]
		\normalsize \textsuperscript{1}ScaDS.AI Dresden/Leipzig, Technische Universit\"at Dresden, Dresden, Germany\\
		\normalsize \textsuperscript{2}Department of Mathematics, KNUST, Kumasi, Ghana\\
		\normalsize \textsuperscript{3}alphaspeech - c/o alpha NT GmbH, Dresden, Germany\\
		\normalsize \textsuperscript{4}Amazon, Technische Universit\"at Dresden, 
		Dresden, Germany\\
		\normalsize \textsuperscript{5} TIB – Leibniz Information Centre for Science and Technology, Hannover, Germany\\	[0.35em]
		\normalsize \texttt{\{kossi.amouzouvi,jens.lehmann,sahar.vahdati\}@tu-dresden.de}, \texttt{robert.wardenga@alphaspeech.de}, \texttt{sahar.vahdati@tib.eu}\\
		\normalsize \textsuperscript{*}Equal contribution; **  Work done outside of Amazon
	}
}
\date{}

\begin{document}
	\maketitle
	
	\begin{abstract}
		Foundation models excel at language, where sentences become tokens, and vision, where images become pixels, because both reduce to discrete symbols on a shared, fixed grid. Knowledge Graphs share the discreteness, but not the geometry. Their entities and relations are discrete symbols, yet their arrangement is relational and lacks a common, fixed grid. Knowledge Graphs (KGs) share the discreteness, but not the geometry.
		They form irregular, non-Euclidean topologies whose local neighborhoods differ from graph to graph. Therefore, Knowledge Graph Foundation Models (KGFMs) rely on identifying structural invariances to produce transferable representations. Without a universal token set, KGFMs are limited in their ability to transfer representations across unseen KGs. We close this gap by treating graphlets, small connected graphs, as structural tokens that recur in heterogeneous KGs. In this paper, We introduce a model-agnostic framework based on a vocabulary of graphlets that mines a KG between relations via pattern matching. In particular, we considered closed and open 2- and 3-path, and star graphlets, to obtain robust invariances. The framework is evaluated on 51 KGs from a wide range of domains, for zero-shot inductive and transductive link prediction. Experiments show that adding simple graphlets to the vocabulary yields models that outperform prior KGFMs. 
	\end{abstract}

	\section{Introduction}
	Recently, Large language Models (LLMs), have garnered significant attention for their remarkable natural language understanding capabilities \citep{bommasani2021opportunities, zhao2023survey, wei2022emergent}. These models are pretrained on massive corpora of diverse text data \citep{chang2024survey}, allowing them to learn not only the syntax and grammar of language but also the semantics and contextual usage of words and phrases. 
	Despite differences in architecture (e.g., transformer-based, decoder-only, encoder-decoder), all LLMs operate on tokens; basic units of text that may be whole words or subwords. 
	During training, the models construct a universal vocabulary of tokens. 
	This token-based processing enables LLMs to generalize effectively to new words by breaking them down into familiar token components. In turn, complete sentences can be reconstructed from these tokens. 
	As a result, LLMs achieve strong generalization across languages \citep{lin2021few}, domains~\citep{pan2024unifying}, and  \citep{
		brown2020language, wang2022language}. 
	The (unstructured) textual data can be saved as tripled-based data, known as Knowledge Graphs  ~\citep{hogan2021knowledge, noy2019industry, ehrlinger2016towards}.  
	Knowledge Graph embeddings (KGEs) are a class of representation learning models specialized for Knowledge Graphs (KGs). They learn entity and relation embeddings 
	based on their labeled identities and the structure of the triples they participate in~\citep{wang2017knowledge}. While effective at modeling relational patterns, they are limited in their ability to generalize to unseen entities or relations. 
	
	In contrast to LLMs, KGEs do not capture any natural understanding of the labels themselves. As a result, adapting KGEs to new entities or relation types typically requires retraining from scratch on augmented data~\citep{hamaguchi2017knowledge, teru2020inductive}. 
	To overcome this limitation, recent approaches explore structure-driven generalization in ~\citep{liu2021indigo}. One idea is to treat structural patterns in a graph analogously to how LLMs treat tokens in text data. These patterns capture local and global structural invariants independent of specific entity labels or relation types.
	\newline
	\textbf{Motivating Example.} Let us consider the three illustrative KGs: Family, Corporate and Scholarly, shown in Figure~\ref{fig:vert}. Despite the fact that the type of relations and entity labels are not the same, their underlying structural topology is the same. 
	This allows for a mapping between their relations (\emph{grand\_father\_of} $\leftrightarrow$ \emph{grandvisor}, \emph{son\_of} $\leftrightarrow$ \emph{mentor\_of}, \emph{wife\_of} $\leftrightarrow$ \emph{cofounder\_with})  enabling structure-level transfer. 
	This core insight forms the basis of KGFMs. 
	\begin{figure*}[t] 
		\centering
		\vspace{0.5\baselineskip}
		\includegraphics[width=1\linewidth,trim=0 10 0 5]{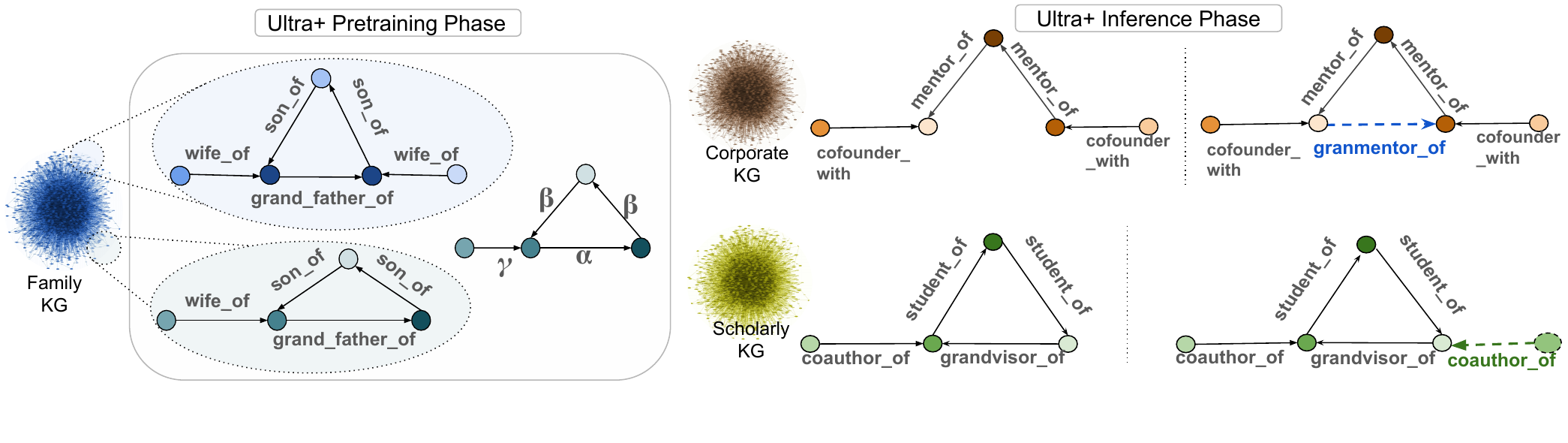}
		\caption{The KGFM model $\Ultra$, pretrained on a large collection of KGs, including the Family KG, recognizes the Corporate, and Academic KGs as instances of the same graphlet patterns.}
		\label{fig:vert}
	\end{figure*}
	Rather than embedding labels, KGFMS (\citeauthor{galkin2023towards}, ~\citeauthor{huang2025expressive}) learn to reason over vocabulary of relations forming relational subgraphs. 
	Group of ordered relations that co-occur within a specific type of subgraph is referred to as an occurrence of a graphlet. These graphlets form the core of the structural vocabulary used to construct a relation graph, where relations become nodes, and the graphlets define typed edges between them. 
	As shown in Figure~\ref{fig:vert}, a KGFM would detect and save the pattern formed by the cycle $(\gamma, \alpha, \beta, \beta)$ as part of its vocabulary. This learned pattern can then be used to infer missing facts such as the triple  \emph{mentor\_of}(A. Einstein, ?). 
	However, existing KGFMs have limitations. First, they fail to distinguish between closed and open paths, treating all subgraphs of similar size as equally informative. Second, a single occurrence of a graphlet is often enough to connect involved relations in the relation graph, which may lead to decreased robustness.
	
	\textbf{Our Contributions.} 
	The key contributions of our model $\Ultra$ are  as follows
	
	(i) {\bf Query-based relation graph extraction}: We propose a flexible SPARQL-based extraction method that efficiently identifies informative structures without relying on sparse matrix multiplication (SPMM); 
	
	(ii) {\bf Closed and open relations}: We incorporate both closed and open graphlets to capture a wider range of relational patterns; 
	
	(iii) {\bf Binary relation}: We represent graphlets as positional binary relations than traditional n-ary relations; and
	
	(iv) {\bf Model-agnostic design}: Ultra+ is modular and can be integrated into any KGFM that uses relation graphs or structural vocabulary, making it 
	adaptable across KGFM architectures. 
	
	$\Ultra$ substantially strengthens how KGFMs capture and represent complex structural patterns. 
	Its main objective is not to propose a new model architecture, but to boost performance and create more robust KGFMs by enriching their structural vocabulary.

	\section{Related Work}
	A KGE model learns the vector representations  of relations ($\mf{r}$) and entities ($\mf{E}$) of KGs via parametrized relation-specific transformation functions, $\Phi_{\mf{r}}: $$\mf{h} $$\mapsto $$\Phi_{\mf{r}}(\mf{h}),$ 
	over the entity embedding space.  
	KGEs can be categorized based on their underlying principles such as geometric transformations, tensor decomposition, deep neural architectures, or foundational graph approaches. 
	
	\textbf{Geometric and Tensor Decomposition Models.} 
	These models can be grouped into three main types. The first are translational-based models such as $\mathrm{TransE}$~\citep{bordes2013translating} $, \mathrm{TransH}$ ~\citep{wang2014knowledge}, and $\rmm{TransR}$~\citep{lin2015learning}, which embed entities and relations in real vector space and use identity mappings ($\Phi_{\mathbf{r}} = \mathbf{r}$). $\mathrm{TransH}$ and $\mathrm{TransR}$ add extra relation embeddings, 
	to improve modeling capacity. 
	While effective for link prediction, these models struggle with relational patterns like closed paths. 
	To address these limitations, rotation-based models such as $\mathrm{RotatE}$ ~\citep{sun2019rotate} were introduced. $\mathrm{RotatE}$ represents relations as rotations: $\Phi_{\mathbf{r}}(\mathbf{h}) = e^{i\theta_r}\mathbf{h}$ with $\theta_r \in (-\pi, \pi]^d$. This shift from real to complex algebra enables better modeling of relational patterns.  
	However, these models are not generalizable to new entities and relations. 
	
	\textbf{Deep Neural Network Models.} These models leverage deep learning to extract graph representations. A primary category includes Graph Neural Networks (GNNs), especially graph convolutional models that iteratively aggregate information from neighboring nodes. 
	A pioneer work is  Relational GNC~\citep{schlichtkrull2018modeling}, an encoder-decoder model. The encoder of {R-GCN} learned latent embedding of entities, which are passed to the decoder based on DistMult, a tensor decomposition model. 
	However, R-GCNs does not learn relation embeddings. To address this limitation, {TransGCN}, {RotatEGCN} \citep{cai2019transgcn}, and {ComplexGCN}~\citep{ZEB2022116796} integrate GCNs with 
	geometric KGE models like {TransE}, {RotatE}, and {ComplEx} to jointly learn entity and relation embeddings and capture richer structural semantics. 
	End-to-end trainable GNNs are restricted to a single KG downstream task and cannot generalize to new KGs without retraining.
	
	\textbf{Knowledge Graph Foundation Models.} 
	KGFMs overcome these limitations by enabling pretrained GNNs or LLMs to inductively generalize to new KGs in zero or few-shot paradigms~\citep{wang2025graph, Liu2023TowardsBeyond}.  {ULTRA}~\citep{galkin2023towards}, a KGFM for KG reasoning, 
	constructs a relation graph whose nodes are the relations from the original KG, and edges represent the connections between relations in paths of length two. 
	Leveraging on the invariance of relational structure across datasets,  
	a labeling trick, and conditional representations on both relations and entities, 
	{Ultra} enhances the generalization of KG reasoning. {ULTRAQuery}~\citep{galkin2024zero} exploits the ability of {Ultra} in KG reasoning to find potential missing links, and uses non-parametric fuzzy logic operators to answer complex questions. 
	{AnyGraph} ~\citep{xia2024anygraphgraphfoundationmodel} overcomes the limit of {Ultra}, by generalizing to in- and cross- domain link prediction, and node and graph classification tasks. 
	The key factor behind the success of existing KGFMs lies in the  construction of suitable graph vocabularies~\citep{Mao2024Position:Here}, i.e. basic transferable units that underlie graphs. 
	While models such as Mole-BERT relies on context-aware atom vocabulary~\citep{xia2023molebert} for molecule graph classification, {Ultra} and {Motif}~\citep{huang2025expressive} rely on paths of length two and motifs to define graph vocabulary, respectively. 
	However, these approaches overlook closed paths, which are prevalent and essential.
	
	In this paper, we extend the {Ultra} framework by introducing a novel graphlet-based vocabulary for KG reasoning. 
	Unlike Ultra, our framework explicitly encodes cyclic structures, this allows us to capture richer structural patterns beyond simple paths of length two. 
	Moreover, in contrast to {Motif}, our vocabulary supports higher-order interaction patterns via binary relations within a standard relation graph, rather than relying on n-ary relations in a relation hypergraph. This ensures our relation graph remains a simple KG, preserving compatibility with most established KG processing techniques. 
	
	\section{Preliminaries}\label{sec:Preliminaries}
	\subsection{Inductive Knowledge Graph Embeddings}\label{sec:Knowledge Graphs}
	\vspace{-0.3cm}
	We consider a \emph{Knowledge Graph} as a multi-relational directed graph $ K =(\mathcal{E}^K, \mathcal{R}^K, \mathcal{T}^K_+)$ where $\mathcal{E}^K, \mathcal{R}^K,$ and $\mathcal{T}_+^K$ are the set of nodes (entities), edge labels (relations), and  
	ordered pairs (edges or triples) formed as \emph{relation(head entity, tail entity)} respectively. We refer to the head and tail entity as $h$ and $t$ or $e$ in general, and to the relation as $r$ or $q$. 
	$\mathcal{T}^K_+$ is a subset of $ \mathcal{T}^K$ which contains all plausible triples; and $\mathcal{T}^K_- 
	= \mathcal{T}^K \setminus \mathcal{T}^K_+$ is the set of 
	corrupted triples. 
	Since all entities in $\mathcal{E}^K$ are used in constructing $\mathcal{T}^K_+,$ corrupted triples (used as negative samples)
	result from replacing the head or the tail entity of true triples. 
	The set $\mathcal{N}(r, t) = \set{h \vert r(h, t) \in \cT^+}$ is called the neighborhood of $t$. 
	
	A \emph{relation graph} is constructed from the KG by examining how relations within a target KG appear collectively in subgraphs. The labels on its edges and nodes originate from subgraph configurations and relations within the target KG. 
	The relations in a relation graph can also be referred to as \emph{meta-relations} (see Section \ref{sec: A Structural Vocabulary for KGFMs} for more details). 
	
	KGE models pretrained on $K$ are evaluated on a test KG $K_{test} = (\cE^K_{test}, \cR^K_{test}, \cT^K_{test+})$  
	to predict missing links. 
	Entities in $\cE^K_{test} \setminus \cE^K$ and relations in $\cR^K_{test} \setminus \cR^K$ are called unseen entities and relations, respectively. 
	A KGE model is a \emph{transductive} model if $\cE^K_{test} \subseteq \cE^K$ and $\cR^K_{test} \subseteq \cR^K,$ an \emph{inductive} model otherwise. 
	\emph{Zero-Shot Link Prediction} (ZSLP) involves evaluating models pretrained on $K$, directly on $K_{test}$. 
	Inductive ZSLP can be categorized into three main tasks: \emph{Relation learning} involving predicting unseen relations (Ind.(r)), \emph{entity learning} focusing on predicting facts involving unseen entities (Ind.(e)), and \emph{graph transfer} requiring generalization to both unseen entities and relations (Ind.(e,r)).
	
	\subsection{Knowledge Graph Homomorphism}
	A key feature of our framework is its ability to detect occurrences of specific subgraph patterns within a KG by leveraging graph homomorphisms for structural matching.
	
	A \emph{KG homomorphism} is a structure preserving mapping between two KGs. 
	It consists of entity and relation mappings to relate entities and relations from one KG to the other. Thus, $\phi : K \to K'$ is a KG homomorphism if there exists two mappings 
	$\eta: \cE^K \to \cE^{K'}$ and $\rho: \cR^K \to \cR^{K'}$ so that the product function $\eta\cdot\rho\cdot\eta$ maps any triple $r(h, t) \in \cT^{K}$ to $\phi(r(h, t)) $ = $ \rho(r)(\eta(h), \eta(t)) $ $\in \cT^{K'}$. 
	$\phi(K) = \left( \eta(\cE^K), \rho(\cR^K), \phi(\cT^K)\right)$, the \emph{image KG} of $K$ by $\phi,$ is a subgraph of $K'.$ 
	A homomorphism $\phi$ is said to be \emph{injective} if $\eta$ and $\rho$ are both injective. An injective KG homomorphism $\phi$ is called \emph{KG monomorphism} and its mapping is represented by $\phi: K \xrightarrow[]{mon} K'.$

	\subsection{Graphlets and Motifs}
	The objective of our proposed method is to represent relational invariances and develop a structural vocabulary. To this end, the method identifies different graphlet occurrences within a KG. 
	Usually graphlets are defined as induced subgraphs with respect to a Knowledge Graph while motifs are graphlets that occur frequently in a given KG \citep{prvzulj2007biological, milo2002network, ribeiro2021survey}. 
	To represent relational patterns in a more nuanced way, the following definitions are employed.
	
	\begin{definition}[Graphlet]\label{def:graphlet}
		A graphlet $\cG = (G, \mgg)$ is a (small) connected Knowledge Graph $G=(\cE, \cR, \cT)$ with an order $\mgg$ on $\cR.$ 
		$\mgg(\cR)=\mgg\left(r_{1}, r_{2}, \ldots, r_{m}\right)$ 
		a tuple defining the order on $\cR\ni r_i.$ 
		The cardinality of a graphlet $\cG$ is the number of its edges, $card(\cG) = \norm{\cT}$.
	\end{definition}
	The order specifies the manner in which directed relations within $\cR$ connect their endpoints to create the graph structure.  
	The order can relate two or many relations in $\cR$; it is called \emph{binary} and \emph{n-ary} respectively. 
	The arity ($n$) can be reduced from $n$ to $m$ with $m\leq n$ by only considering $m$ relations as arguments and the $n-m$ remaining relations become dummy arguments. 
	This order is referred to as a \emph{positional m-ary order}; it is denoted by $\mgg_{i_1, \ldots, i_m \leq n}$ when the attention is to stress on the position of the arguments and the initial arity. 
	Thus $\mgg_{1, 2, 3 \leq 3}$ is an ordinary 3-ary order whereas $\mgg_{1, 3 \leq 3}$ is a positional binary order. 
	\begin{theorem}
		\label{theo:pos order spans n-ary}
		Any positional m-ary order, $m < n,$ spans a group of n-ary orders.  
	\end{theorem}
	In particular, the tuple $\mgg_{1,3\leq 3}(\cR)$ could be induced by any of $\mgg_{1,1,3}(\cR), \mgg_{1,3,3}(\cR)$ and $\mgg_{1,2,3}(\cR).$ In other words, if  $\mgg_{1,3\leq 3}(\cR)$ does not exist, then neither $\mgg_{1,1,3}(\cR), \mgg_{1,3,3}(\cR)$ nor $\mgg_{1,2,3}(\cR)$ exist.

	\begin{definition}[Graphlet occurrence]\label{def:graphlet occurrence} Let $K=(\cE^K,\cR^K,\cT^K)$ be a KG, and $\mgg$ a positional m-aray order. 
		A graphlet $\cG = (G, \mgg)$ \emph{occurs}  in $K$ if and only if $G
		$ is monomorphic to $K$, and the monomorphism $\phi_{\mgg}: G\xrightarrow{mon} K$  is an order-preserving mapping; 
		that is, $\rho$ maps $\mgg(\cR)$ to $\mgg\circ\rho (\cR)=\mgg\left(\rho(r_{1}), \rho(r_{2}), \ldots, \rho(r_{m})\right).$ The tuple $\mgg(\cR)$ induces the set of tuples $\mgg(\cR^K) = \set{\mgg\circ\rho (\cR) \vert \rho:\cR \xrightarrow{mon}\cR^K}$.  
		The image graph $\phi_{\mgg}(G),$  is called an \emph{occurrence} of $\cG$ in $K$. The set of occurrences of $\cG$ in $K$ is denoted and defined by  
		$\cG(K) = \set{\phi_{\mgg}(G) \vert \phi_{\mgg}: G \xrightarrow{mon} K}$. 
		Two occurring subgraphs $\phi(G)$ and $\phi'(G)$ are said \emph{equivalent}, $\phi(G)\equiv_{\mgg}\phi'(G),$ if $\rho(r_{i_j}) = \rho'(r_{i_j}) $ for ${j} \leq m, i_j \leq n, $ and $r_{i_j}\in \cR.$ 
		The equivalence class denoted by $\overline{\phi(G)} = \set{\phi'(G)\vert \phi'(G)\equiv_{\mgg}\phi(G) }$ is the set of subgraph occurrences equivalent to $\phi(G)$ and $\overline{\cG(K)} = \set{\overline{\phi(G)} \vert \phi(G) \in \cG(K)}$ is the set of all equivalence classes.
	\end{definition}
	The fact that two classes $\overline{\phi(G)}$ and $\overline{\phi'(G)}$ are different if and only if $\rho(\cR) \neq \rho'(\cR)$ implies that there is a one-to-one correspondence between $\overline{\cG(K)}$ and $\mgg(\cR^K).$ 
	We can therefore represent an equivalence class by the class $\overline{\phi(G)}$ or the tuple $\mgg\circ\rho(\cR).$

	The set of all graphlets of cardinality less than four is displayed in Figure~\ref{fig:graphlets_ord3}. 
	Graphlets play the role of the smallest structural entity in a Graph and are therefore well suited to investigate the local and global structure of a KG. 
	
	\begin{figure*}
		\centering
		\begin{tikzpicture}[>=Stealth, node distance=0.8cm,
			every node/.style={inner sep=1.5pt},
			scale=1]
			
			\tikzset{
				headEdge/.style={
					->,
					draw=ForestGreen!80!black,
					line width=1pt,
					preaction={draw=gray!40, line width=3pt}
				},
				goldEdge/.style={
					->,
					draw=Goldenrod!80!black,
					line width=1pt
				},
				nodeStyle/.style={
					circle,
					fill=black,
					inner sep=0pt,
					minimum size=2pt
				}
			}
			
			\triad{0}{3}{
				\draw[->] (a) -- (b);
				\node[nodeStyle] at (a) {};
				\node[nodeStyle] at (b) {};
			}{(a1): $\rmm{f}$}
			
			\triad{2}{3}{
				\draw[->] (a) edge[bend left] (b);
				\draw[->] (b) edge[bend left] (a);
				\node[nodeStyle] at (a) {};
				\node[nodeStyle] at (b) {};
			}{(b1): $\rmm{ff_c}$}
			
			\triad{4}{3}{
				\draw[->] (a) edge[bend left] (b);
				\draw[->] (a) edge[bend right] (b);
				\node[nodeStyle] at (a) {};
				\node[nodeStyle] at (b) {};
			}{(c1): $\rmm{fr_c}$}
			
			\triad{6}{3}{
				\draw[->] (a) -- (b);
				\draw[->] (b) -- (c);
				\node[nodeStyle] at (a) {};
				\node[nodeStyle] at (b) {};
				\node[nodeStyle] at (c) {};
			}{(d1): $\rmm{ff_o}$}
			
			\triad{8}{3}{
				\draw[headEdge] (b) -- (a);
				\draw[headEdge] (b) -- (c);
				\draw[->] (a) -- (b);
				\draw[->] (c) -- (b);
				\node[nodeStyle] at (a) {};
				\node[nodeStyle] at (b) {};
				\node[nodeStyle] at (c) {};
			}{(e1): $\rmm{fr_o}$, \textcolor{ForestGreen!80!black}{$\rmm{rf_o}$}}
			
			\triad{10}{3}{
				\draw[->] (a) -- (b);
				\draw[->] (b) -- (c);
				\draw[->] (c) -- (a);
				\node[nodeStyle] at (a) {};
				\node[nodeStyle] at (b) {};
				\node[nodeStyle] at (c) {};
			}{(f1): $\rmm{fff_c}$} 
			
			\triad{12}{3}{
				\draw[->] (a) -- (b);
				\draw[->] (b) -- (c);
				\draw[->] (a) -- (c);
				\node[nodeStyle] at (a) {};
				\node[nodeStyle] at (b) {};
				\node[nodeStyle] at (c) {};
			}{(g1): $\rmm{ffr_c}$} 
			
			\triadT{0}{1.25}{
				\draw[->] (a) -- (b);
				\draw[->] (b) -- (c);
				\draw[->] (c) -- (d);
				\node[nodeStyle] at (a) {};
				\node[nodeStyle] at (b) {};
				\node[nodeStyle] at (c) {};
				\node[nodeStyle] at (d) {};
			}{(a2): $\rmm{fff_o}$} 
			
			\triadT{2}{1.25}{
				\draw[headEdge] (b) -- (a);
				\draw[headEdge] (c) -- (b);
				\draw[headEdge] (c) -- (d);
				\draw[->] (a) -- (b);
				\draw[->] (b) -- (c);
				\draw[->] (d) -- (c);
				\node[nodeStyle] at (a) {};
				\node[nodeStyle] at (b) {};
				\node[nodeStyle] at (c) {};
				\node[nodeStyle] at (d) {};
			}{(b2): $\rmm{fff_o}$, \textcolor{ForestGreen!80!black}{$\rmm{rff_o}$}} 
			
			\triadT{4}{1.25}{
				\draw[->] (a) -- (b);
				\draw[->] (c) -- (b);
				\draw[->] (c) -- (d);
				\node[nodeStyle] at (a) {};
				\node[nodeStyle] at (b) {};
				\node[nodeStyle] at (c) {};
				\node[nodeStyle] at (d) {};
			}{(c2): $\rmm{frf_o}$} 
			
			\triadT{6}{1.25}{
				\draw[->] (a) -- (c);
				\draw[->] (b) -- (c);
				\draw[goldEdge] (d) -- (c);
				\node[nodeStyle] at (a) {};
				\node[nodeStyle] at (b) {};
				\node[nodeStyle] at (c) {};
				\node[nodeStyle] at (d) {};
			}{(d2): $\rMany{f}{r}{1,}{2}$} 
			
			\triadT{8}{1.25}{
				\draw[->] (c) -- (a);
				\draw[->] (c) -- (b);
				\draw[goldEdge] (d) -- (c);
				\node[nodeStyle] at (a) {};
				\node[nodeStyle] at (b) {};
				\node[nodeStyle] at (c) {};
				\node[nodeStyle] at (d) {};
			}{(e2): $\rMany{f}{f}{1,}{2}$} 
			
			\triadT{10}{1.25}{
				\draw[->] (c) -- (a);
				\draw[->] (c) -- (b);
				\draw[goldEdge] (c) -- (d);
				\node[nodeStyle] at (a) {};
				\node[nodeStyle] at (b) {};
				\node[nodeStyle] at (c) {};
				\node[nodeStyle] at (d) {};
			}{(f2): $\rMany{r}{f}{1,}{2}$} 
			
			\triadT{12}{1.25}{
				\draw[->] (a) -- (c);
				\draw[->] (b) -- (c);
				\draw[goldEdge] (c) -- (d);
				\node[nodeStyle] at (a) {};
				\node[nodeStyle] at (b) {};
				\node[nodeStyle] at (c) {};
				\node[nodeStyle] at (d) {};
			}{(g2): $\rMany{r}{r}{1,}{2}$} 
			
			\triadC{0}{-1.25}{
				\draw[goldEdge] (a) -- (c);
				\draw[goldEdge] (b) -- (c);
				\draw[->] (c) -- (d);
				\draw[->] (c) -- (e);
				\node[nodeStyle] at (a) {};
				\node[nodeStyle] at (b) {};
				\node[nodeStyle] at (c) {};
				\node[nodeStyle] at (d) {};
				\node[nodeStyle] at (e) {};
			}{(a3): $\rMany{f}{f}{2,}{2}$} 
			
			\triadC{2}{-1.25}{
				\draw[goldEdge] (a) -- (c);
				\draw[goldEdge] (b) -- (c);
				\draw[->] (d) -- (c);
				\draw[->] (e) -- (c);
				\node[nodeStyle] at (a) {};
				\node[nodeStyle] at (b) {};
				\node[nodeStyle] at (c) {};
				\node[nodeStyle] at (d) {};
				\node[nodeStyle] at (e) {};
			}{(b3): $\rMany{f}{r}{2,}{2}$} 
			
			\triadC{4}{-1.25}{
				\draw[goldEdge] (c) -- (a);
				\draw[goldEdge] (c) -- (b);
				\draw[->] (c) -- (d);
				\draw[->] (c) -- (e);
				\node[nodeStyle] at (a) {};
				\node[nodeStyle] at (b) {};
				\node[nodeStyle] at (c) {};
				\node[nodeStyle] at (d) {};
				\node[nodeStyle] at (e) {};
			}{(c3): $\rMany{r}{f}{2,}{2}$} 
			
			\triadC{6}{-1.25}{
				\draw[goldEdge] (a) -- (c);
				\draw[goldEdge] (b) -- (c);
				\draw[goldEdge] (d) -- (c);
				\draw[->] (c) -- (e);
				\node[nodeStyle] at (a) {};
				\node[nodeStyle] at (b) {};
				\node[nodeStyle] at (c) {};
				\node[nodeStyle] at (d) {};
				\node[nodeStyle] at (e) {};
			}{(d3): $\rMany{f}{f}{3,}{1}$} 
			
			\triadC{8}{-1.25}{
				\draw[goldEdge] (a) -- (c);
				\draw[goldEdge] (b) -- (c);
				\draw[goldEdge] (d) -- (c);
				\draw[->] (e) -- (c);
				\node[nodeStyle] at (a) {};
				\node[nodeStyle] at (b) {};
				\node[nodeStyle] at (c) {};
				\node[nodeStyle] at (d) {};
				\node[nodeStyle] at (e) {};
			}{(e3): $\rMany{f}{r}{3,}{1}$} 
			
			\triadC{10}{-1.25}{
				\draw[goldEdge] (c) -- (a);
				\draw[goldEdge] (c) -- (b);
				\draw[goldEdge] (c) -- (d);
				\draw[->] (c) -- (e);
				\node[nodeStyle] at (a) {};
				\node[nodeStyle] at (b) {};
				\node[nodeStyle] at (c) {};
				\node[nodeStyle] at (d) {};
				\node[nodeStyle] at (e) {};
			}{(f3): $\rMany{r}{f}{3,}{1}$} 
			
			\triadC{12}{-1.25}{
				\draw (e) edge[bend left]  (d);
				\draw (e) edge[bend right] (d);
				\draw (a) -- (d);
				\node[nodeStyle] at (a) {};
				\node[nodeStyle] at (d) {};
				\node[nodeStyle] at (e) {};
			}{(g3): $**_{*:*}$} 
		\end{tikzpicture}
		
		\caption{\textbf{Graphlets of size less than 5.} 
			$\rmm{f}$ and $\rmm{r}$ denote forward and reverse edges, and subscripts $\rmm{c}$ and $\rmm{o}$ indicate closed and open paths. 
			The \textcolor{ForestGreen!80!black}{green head arrows} (shown with a light gray halo for clarity) form alternative graphlets, which are also indicated by the green labels to the right of the black text labels. 
			The \textcolor{Goldenrod!80!black}{golden arrows}, together with the black arrows, form distinct topological graphlets. 
			Each vertex is marked with a small filled node for readability. 
			The last four graphlets shown in the third column are not included in our approach.}
		\label{fig:graphlets_ord3}
	\end{figure*}
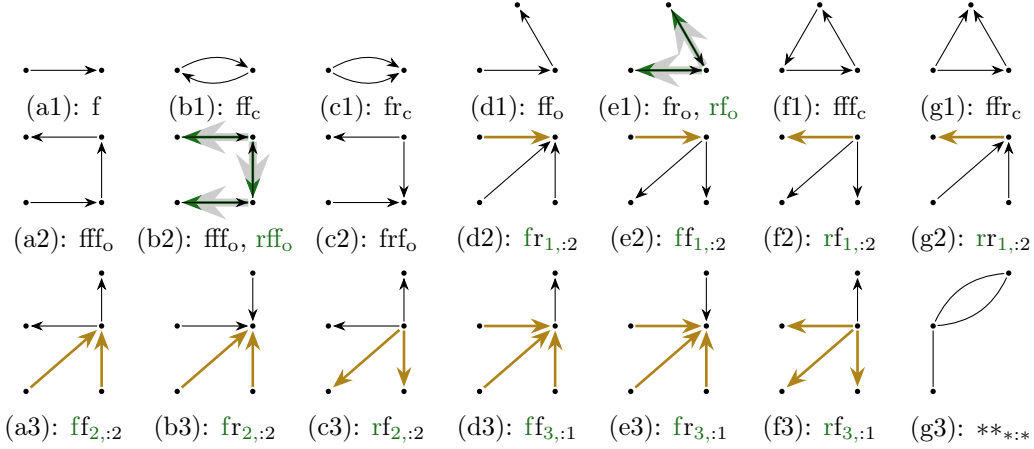

	\section{Method}
	Our model $\Ultra$ is a generalized extension of the $\mathrm{Ultra}$ framework introduced by \citeauthor{galkin2023towards}, advancing its capabilities in relational pattern learning for KG reasoning. 
	While the original $\mathrm{Ultra}$ relies solely on length-2 paths to define relational dependencies, $\Ultra$ extends this approach by incorporating a richer set of graphlet-based patterns, capturing more complex and higher-order interactions between relations. 
	In contrast to $\Motif$, 
	$\Ultra$ constructs a binary relation graph using graphlets induced by positional binary orders, thereby preserving pairwise semantics. This shift allows $\Ultra$ to encode cycles and subgraph patterns without resorting to hypergraph complexity. 
	
	\subsection{A Structural Vocabulary for Knowledge Graph Foundation models}\label{sec: A Structural Vocabulary for KGFMs}
	The structural vocabulary used to construct the relation graph constitutes the fundamental basis of $\Ultra$.
	
	\begin{definition}\label{def:struct-vocab}
		A structural vocabulary over a KG, $K$, is a finite set $\cV = \set{(G_i, \mgg_i), i\leq n_V}$ of graphlets, and a weighting function 
		\begin{equation}
			\omega: \bigcup_{i} \overline{\cG_i(K)} \to \mathbb{N}, \quad \overline{\phi_{\mgg}(G)} \mapsto \lvert \overline{\phi_{\mgg}(G)} \rvert
		\end{equation}    
		mapping  equivalence classes of occurrences to their cardinalities. The Knowledge Graph $K$ can be a union of KGs, that is $K = \cup_k (\cE^k, \cR^k, \cT^k),$ and the domain of $\omega$ becomes $\bigcup_{i, k} \overline{\cG_i(K_k)}$.
	\end{definition}
	
	\begin{definition}
		Let $K=(\cE^K, \cR^K, \cT^K)$ be a KG, and $\cV$ a structural vocabulary over $K$. 
		We denote the \emph{relation graph} over $K$ upon the structural vocabulary $\cV$ by $\fK =(\fE, \fR, \fT)$ 
		where the set of nodes $\fE=\cR^K$ and meta-relations $\fR = \cV$. 
	\end{definition}
	A structural vocabulary of binary orders yields relation graphs; conversely, relation hypergraphs are generated. 
	The (hyper)edges in $\fT$ are the tuples $\mgg_i\circ\rho(\cR)$ .  
	The tuple $\mgg\circ\rho(\cR)$ does not exist unless its weight is nonzero. We say $\mgg$ is an $\varepsilon$-edge between the relations $\rho(r_1), \cdots, \rho(r_m) \in \cR^K$, if and only if $\omega(\mgg\circ\rho(\cR)) = \varepsilon.$ 
	To define the structural vocabulary, we did not specify any graphlet. This ensures that our framework can accommodate any type of graphlet. $\Ultra$'s structural vocabulary is restricted to (short) path and topology based vocabularies.

	\textbf{Path-Based Vocabulary.} 
	In a graph, $K,$ a path of length $p$, $\cP_p = \set{r_i(e_i, e'_i), i=1, \cdots, p}$ is a subset of $\cT^K$ such that two consecutive triples share one entity. 
	These paths are occurrences of the nine graphlets shown in Figures~\ref{fig:graphlets_ord3}(b1--c2). 
	We shall note that $\rmm{rr_o} (r_1, r_2) $ = $ \set{r_1(e_2, e_1), r_2(e_3, e_2)\in \cT^K} $ = $ \rmm{ff_o}(r_2,r_1).$ That is, $\rmm{rr_o}$ can be substituted by $\rmm{ff_o}$. 
	In summary, the $\cP_2$-based vocabulary $\cV_2 = \set{\rmm{ff_{o, c}, fr_{o, c}, rf_o}}$ is sufficient to characterize all closed or open 2-paths in $K.$ 
	In general, we define two distinct families of $\cP_p$-based vocabularies $\cV_p $ = $ \big\{\rmm{u_{\smile} v_z}\vert \rmm{u, v}\in \set{\rmm{f, r}},$ $  _{\smile}\in \set{\rmm{f, r}}^{i},$ $ i = 0, \ldots, p-2, $ $\rmm{z\in\set{o, c}}\big\}$ and $\cU_p $ = $ \big\{\rmm{u_{\smile} v}\vert \rmm{u, v}\in \set{\rmm{f, r}},$ $  _{\smile}\in \set{\rmm{f, r}}^{i},$ $ i = 0, \ldots, p-2, \big\}.$ $_{\smile}$ is any sequence of length $i$ over the alphabets $\rmm{f, r}.$ $\rmm{u_{\smile}v_z}$ are positional binary orders relating the first and last relations appearing in a path of length $p>1.$ For $p=2, \rmm{u_{\smile}v_z = uv_z}$ and for $p=3,$ $ \rmm{u_{\smile}w_z}  := $ $ \rmm{uvw_z} \in \big\{\rmm{fff_z, ffr_z, frf_z, rff_z}$ $\vert \rmm{z\in\set{o, c}}\big\}.$  
	It follows that $\cV_m \subset \cV_n$ if $m \leq n.$ This remains valid for the $\cU_m.$ It is imperative to note that the $\rmm{u_{\smile}v}$ positional binary orders are incapable of discerning between closed and open paths. We designed variants of $\Ultra$ using these two families of structural vocabularies.

	\textbf{Topology-based Vocabulary.} 
	The degree of an entity in a KG is the sum of incoming and outgoing relations.  The average number of degree per entity informs on the sparsity or density of the KGs. 
	The type of relations surrounding an entity allow us to extract a subgraph centered on that entity, called an \emph{m-star}, where $m$ is the degree of that entity.  
	These m-stars are occurrences of graphlets that form the topology based vocabulary, denoted by $\mathcal{M}_m.$  
	In m-stars, we count how many times each relation appears around the centered entity. For two relations, we have $i+j = m > 2$ and any of the relation can be an ingoing or outgoing relation. We write $\mathcal{M}_{ij},$ to emphasize on the degree of the relations. 
	Figure~\ref{fig:graphlets_ord3} depicts $\mc{M}_{12} = \set{uv_{1:2}\mid u, v\set{f, r}} = \mc{M}_{21} = \mc{M}_{3}, \mc{M}'_{4} = \mc{M}_{3}\cup \mc{M}_{22}$ and $\mc{M}_{4} = \mc{M}_{3}\cup \mc{M}'_{4}$. We combined the $\cV_2$ and the $\cV_3$ with the $\mc{M}_4'$ vocabularies to design two variants of $\Ultra.$

	\subsection{Representation Learning}
	KG representation learning consists of learning the entity and relation embeddings while preserving the KG structure. 
	In our context, relations in $\cR$ are entities in $\fR.$ This duality leads to two representations, as described below. 
	
	\textbf{Relation Embedding.} 
	$\Ultra$ embeds relations (nodes of the relation graph $\fK$) into $d_L$ dimensional real vectors, $\Enc{r}{q}{L},$ by an L-layer message passing GNN, $GNN_{\cG}$ .  
	Following \citep{galkin2023towards}, 
	$\Ultra$ conditioned ($r\vert q$) the embedding of relations $r$ on the query triple $q(h, ?).$  The input layer is initialized to $\Enc{r}{q}{0} = \delta_{r=q}\mf{1}^{d_0},$ where $\delta_{r=q} = 1$ if $r = q,$ and $0$ otherwise.  
	The following iterative process defines how the upcoming layers compute the embeddings 
	\small{
		\begin{align*}
			\Enc{r}{q}{t+1} =& \rmm{UP}\big(\Enc{r}{q}{t}, \rmm{AGG}\big[ \big\{\\
			&\rmm{MSG}\big(\big\{\big(\Enc{r'}{q}{t}, \mf{u_{\smile}v_z}\big) 
			\vert r' \in \mathcal{N}(\rmm{u_{\smile}v_z, r}), u_{\smile}v_z \in \cV\big\}\big)\\
			&\big\}\big]\big)
		\end{align*}
	}
	so that $\EncR{q} = GNN_{\cG}\left(\Theta_{u}, \Theta_{a}, \Theta_{m}, q, \fR \right)\in \R^{\vert \cR\vert\times d_L}$
	is the conditional relation embedding matrix of all relations in $\fE.$  $\rmm{UP, AGG}$ and $\rmm{MSG}$ are \emph{update, aggregation, and message passing functions} and $\Theta_{\rmm{x}}, \rmm{x = u, a, m}$ are their respective parameters.

	\textbf{Entity Embedding.} 
	Entities are first initialized conditionally to the query $q(h, ?)$ and the relation embedding $\mf{q},$ a vector column of $\EncR{q}$. We iteratively embed entities as follows
		\begin{align*}
			\Enc{e}{h, q}{0} =& \delta_{e = h}\mf{q}\\
			\Enc{e}{h, q}{t+1} =& \rmm{UP}\big(\Enc{e}{h, q}{t}, \rmm{AGG}\big[ \big\{\\
			& \rmm{MSG}\big(\big\{\big(\Enc{e'}{h, q}{t}, \rmm{f}^t(\mf{q})\big)
			\vert e' \in \mathcal{N}(r, e), r \in \cR\big\}\big)\big\}\big]\big)\\
			\pi\left(h, q, e\right) =& \mf{w}^{\top}\big(\mf{W}^{L'}\Enc{e}{h, q}{L'} + \mf{b}^{L'}\big) + b.
		\end{align*}
	Motivated by the ability of geometric KGEs to capture complex relational patterns through algebraic transformations, the message-passing function is enriched with non-linear layer-specific relational transformations $\rmm{f}^t$. 
	The transformations, $\rmm{f}^t(\mf{q}) = \mf{W}_2^t\rmm{ReLU}\left(\mf{W}_1^t\mf{q} + \mf{b}^t_1\right) + \mf{b}^t_2$  are 2-layer 
	perceptrons with the $\rmm{ReLU}$ activation function; $\mf{W}$ are matrices, $\mf{b}$ and $\mf{w}$ are vectors, and $b$ is a scalar. 
	The update functions consist of a linear transformation followed by a normalization layer, while aggregation is performed through summation. $\pi(h, q, e)$ is the score of the triple $q(h, e)$. 
	The initialization of relations to vectors of ones, $\mathbf{1}^{d_L}$, or zeros, $\mathbf{0}^{d_L}$, and entities to $\mathbf{q}$ or zero vectors, makes the architecture of $\Ultra$ generalizable to unseen relations and entities during inference. 
	$\Ultra$ uses the binary cross entropy (BCE) loss,
	{\small
		\begin{align*}
			\mathcal{L}_{\rmm{BCE}} = - \frac{1}{\vert \cT_+\vert}\sum_{\tau \in \cT_+}\BRac{ & \log\pi(\tau) 
				+ \frac{1}{n(\tau)}\sum_{i=1}^{n(\tau)}\log\left(1 - \pi(\tau'_i) \right)}
		\end{align*}
	} 
	to measure the difference between predicted probabilities and triple plausibilities. The BCE loss penalizes high score for true triples, $\tau,$ low score for corrupted triples, $\tau'_i$. 
	
	\subsection{Comparing $\Ultra$ and Motif}
	The KGFM $\rmm{Motif}$ uses a variety of motifs (n-ary orders) to construct a relation hypergraph. 
	The experiments in ~\citep{huang2025expressive} are conducted on 2-path, 3-path, and k-star motifs, denoted by $\mc{F}^{path}_k$ and $\mc{F}^{star}_k,$ respectively. 
	It can be observed that $\Motif$ is unable to discriminate between closed and open paths, unlike $\Ultra$. The second significant distinction derived from the orders. 
	To explain this, let us consider the motifs arising from their respective 3-path based structural vocabulary.  The 3-aries in $\Motif$ are named and are equivalent to ours as follows: $\rmm{tfh \sim  fff},$ $\rmm{tft \sim ffr},$ $\rmm{hfh \sim frf}$ and $\rmm{hft \sim rff}.$  As the IKG in Figure~\ref{fig:toyKG} is a directed acyclic graph, the $\rmm{u_{\smile}v}$ are equivalent to $\rmm{u_{\smile}v_o}$. Figure~\ref{fig:relation graph for toy KG} is therefore built using the latter.
	From Figure~\ref{fig:toyKG}, $r_1, r_2$ and $r_3$ are linked by the motif $\rmm{tfh}$ and $(a, r_1, b, r_2, c, r_3, d)$ is the only element in the equivalence class $\rmm{tfh}(r_1, r_2, r_3).$ Similarly $\rmm{tfh}(r_1, r_4, r_3)$ and $\rmm{tfh}(r_1, r_5, r_3)$ are singletons. However, the equivalence class $\rmm{fff_o}(r_1, r_3)$ is the union of $\rmm{tfh}(r_1, r_i, r_3), i=2, 4, 5,$ this is to say $\omega\left( \rmm{fff_o}\left(r_1, r_2, r_3 \right)\right) = \sum_i\omega\left( \rmm{tfh}\left(r_1, r_i, r_3 \right)\right).$  In general, the weights of the equivalence classes induced by $\Ultra$'s 3-path motifs are higher than the $\Motif$'s ones. 
	\begin{theorem}\label{theo:robustness}
		Let $\rho$ be a monomorphism from $\cP_3$ to a graph $K.$ If $\rmm{uvw_o\circ \rho}$ is an $\epsilon$-edge and its corresponding motif in $\Motif$'s vocabulary is an $\epsilon'$-edge, then $\epsilon'\leq \epsilon.$ 
	\end{theorem}
	
	Theorem~\ref{theo:robustness} states that if no edge exists between two relations in the $\Ultra$ relation graph, then they are not connected by a hyper-edge in the $\Motif$ relation hypergraph. 
	This proves the robustness of $\Ultra$. Furthermore, the theorem demonstrates that $\Ultra$ is computationally less demanding than $\Motif.$ This difference in computation appears in the $\rmm{GNN}_{\cG}$'s message function.  
	In order to clarify this statement, let us consider the neighborhoods of $r_3$ in both relation graphs. $\Ultra$ returns $\mc{N}(\rmm{fff_o}, r_3) = \set{r_1}$ while $\Motif$ returns $\mc{N}^1(\rmm{tfh}, r_3) =\emptyset, ~\mc{N}^2(\rmm{tfh}, r_3) =\emptyset$ and $\mc{N}^3(\rmm{tfh}, r_3) = \set{r_1, r_2, r_4, r_5};$ where the upper script $^i$ means $r_3$ appears at the $i$'th position in the hyperedge $\rmm{tfh}$. 
	In comparison to $\mc{N}(\rmm{fff_o}, r_3)$, operations over $\mc{N}^3(\rmm{tfh}, r_3)$ result in an increase in compute time. 
	The choice of relation embedding GNN also contributes to the increases of computing time. The HCNets used by $\Motif$ genuinely involves more computation, as it employs a learnable query vector and a sinusoidal positional encoding for each query relation~$q$. 
	
	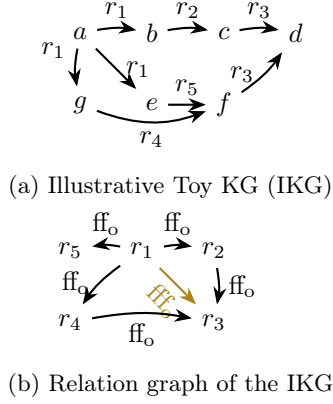
\begin{figure}
		\centering
		\begin{subfigure}{0.45\textwidth}
			\centering
			\begin{tikzpicture}[
				node distance=.95cm,
				every node/.style={},
				->, >=Stealth, thick,
				bend angle=10
				]
				\node (e1) {$a$};
				\node[right of=e1] (e2) {$b$};
				\node[right of=e2] (e3) {$c$};
				\node[right of=e3] (e4) {$d$};
				
				\node[below of=e2] (e6) {$e$};
				\node[below of=e3] (e5) {$f$};
				\node[below of=e1] (e7) {$g$};
				
				\draw (e1) edge[bend left] node[above] {$r_1$} (e2);
				\draw (e2) edge[bend left] node[above] {$r_2$} (e3);
				\draw (e3) edge[bend left] node[above] {$r_3$} (e4);
				\draw (e5) edge[bend right] node[left] {$r_3$} (e4);
				
				\draw (e1) edge node[above, right] {$r_1$} (e6);
				\draw (e6) edge node[above] {$r_5$} (e5);
				
				\draw (e1) edge[bend right] node[above left] {$r_1$} (e7);
				\draw (e7) edge[bend right=20] node[below] {$r_4$} (e5);
			\end{tikzpicture}
			\caption{Illustrative Toy KG (IKG)}
			\label{fig:toyKG}
		\end{subfigure}
		\hfill
		\begin{subfigure}{0.45\textwidth}
			\centering
			\begin{tikzpicture}[
				node distance=.95cm,
				every node/.style={},
				->, >=Stealth, thick,
				bend angle=10
				]
				\node (r1) {$r_1$};
				\node[left of=r1] (r5) {$r_5$};
				\node[below of=r5] (r4) {$r_4$};
				
				\node[right of=r1] (r2) {$r_2$};
				\node[right of=r2] (r6) {};
				\node[below of=r2] (r3) {$r_3$};
				
				\draw (r1) edge[bend right] node[below, left] {$\rmm{ff_o}$} (r4);
				\draw (r1) edge[bend right] node[above] {$\rmm{ff_o}$} (r5);
				\draw (r1) edge[bend left] node[above] {$\rmm{ff_o}$} (r2);
				\draw[Goldenrod!80!black] (r1) edge node[sloped, below] {\textcolor{Goldenrod!80!black}{$\rmm{fff_o}$}} (r3); 
				\draw (r2) edge[bend left] node[above, right] {$\rmm{ff_o}$} (r3);
				\draw (r4) edge[bend left] node[below] {$\rmm{ff_o}$} (r3);
			\end{tikzpicture}
			\caption{Relation graph of the IKG}
			\label{fig:relation graph for toy KG}
		\end{subfigure}
		\caption{(a) A toy Knowledge Graph (IKG) with five relations and seven entities, illustrating the underlying relational structure. (b) The corresponding relation graph constructed from the structural vocabulary of open paths $\{\rm{ff_o}, \textcolor{Goldenrod!80!black}{\rm{fff_o}}\}$, where relations are nodes and edges capture their co-occurrence within paths.}
		\label{fig:toy-graph}
	\end{figure}

	\section{Experiments and Results}
	In our experiments, we aim to answer the following research questions:
	
	\textbf{(RQ1)}  Can the scaling behavior of recent GNN based graph foundation models be improved with the proposed extension?
	\newline 
	\textbf{(RQ2)} Does the zero shot performance increase with increasing vocabulary?  
	\newline
	\textbf{(RQ3)}  Can the addition of specific topological graphlets (e.g., N-M graphlets) help link prediction for containing N-M relations? 
	\newline
	\textbf{(RQ4)} Does enriching vocabulary with closed paths improve model performance? 
	\newline
	\textbf{(RQ5)} Does constructing a relation graph 
	with binary meta-relations offer advantages over using ternary meta-relations?
	
	\tabcolsep=0.143cm
	\begin{table*}[ht!]
		\centering
		\caption{Average zero-shot link prediction MRR and H\@10 over 51 KGs. Baseline results are taken from~\citep{huang2025expressive}. $\cP_n,$ O and C stand for n-, open, and closed paths; and N-M stands for many-to-many subgraphs} 
		\begin{tabular}{l|lll|llllll|| ll}
			\toprule
			\multirow{3}{*}{Model} & \multicolumn{3}{c|}{\multirow{2}{*}{Structural Vocabulary}} & \multicolumn{2}{c}{Ind.$(e)$
			} & \multicolumn{2}{c}{Ind.$(e,r)$} & \multicolumn{2}{c}{Transd.} & \multicolumn{2}{c}{Total Avg.} \\
			&&&& \multicolumn{2}{c}{ (18 KGs)} & \multicolumn{2}{c}{ (23 KGs)} & \multicolumn{2}{c}{(10 KGs)} & \multicolumn{2}{c}{(51 KGs)}\\
			\cline{2-12}
			&$\cV$& Definition & $\# \cV$ & MRR & H\@10 & MRR & H\@10& MRR & H\@10 & MRR & H\@10\\
			\midrule
			\multirow{5}{*}{$\Ultra$}&$\cV_{2}^{-}$&O. $\cP_2$ & $4$ &$.388$ & $.551$ & $.323$ & $.498$&$.338$& $.498$ &$.349$ &$.516$\\
			
			&$\cU_{2}$ &$\cP_2$ & $4$ & $ .425 $&$ .567$&$.350$&$.515$&$.343$&$.499$ &$.375$ &$.530$\\
			
			&$\cV_2$&O.~\&~C.~$\cP_2$ & $ 8$ & \cellcolor{Goldenrod!25}$\underline{.441}$ & $.579$ & $.354$ & \cellcolor{Goldenrod!25}$\underline{.533}$&\cellcolor{Goldenrod!25}$\underline{.349}$& \cellcolor{Goldenrod!25}$\underline{.509}$ & \cellcolor{Goldenrod!25}$\underline{.384}$ &\cellcolor{Goldenrod!25}$\underline{.544}$\\
			
			&$ \cV_2^+$ &O.~\&~C.~$\cP_2$~\&~N-M& $16$ & $.415$ & \cellcolor{RoyalBlue!25}$\textbf{.582}$ & $.349$ & $.525$ & $.347$ & $.504 $ &$.372$ &$.541$\\
			
			&$ \cV_3^-$ &O. $\cP_3$& $16$ & $.423 $ & $.561$ & $.337$ & $.510$ & $.346$ & $.496$ & $.369$ & $.525$\\
			
			& $\cV_3$ &O.~\&~C.~$\cP_3$ & $24$ & \cellcolor{RoyalBlue!25}$\textbf{.445}$ & \cellcolor{Goldenrod!25}$\underline{.581}$ & \cellcolor{Goldenrod!25}$\underline{.355}$ & \cellcolor{RoyalBlue!25}$\textbf{.542}$ & \cellcolor{RoyalBlue!25}$\textbf{.355}$& \cellcolor{RoyalBlue!25}$\textbf{.511}$ & \cellcolor{RoyalBlue!25}$\textbf{.387}$ & \cellcolor{RoyalBlue!25}$\textbf{.549}$\\
			
			&$\cV_3^+$ &O.~\&~C.~$\cP_3$~\&~N-M & $32$ & $.435$ & \cellcolor{Goldenrod!25}$\underline{.581}$ & \cellcolor{RoyalBlue!25}$\textbf{.356}$ & $ .532$ & $.345$ & $.499$ & $.382$ &$.543$\\
			
			\hline
			$\rmm{Ultra}$ & $\cU_{2}$&$\cP_2$ & $4$ & $.431$ & $.566$ & $.345$ & $.513$ & $.339$ & $.494$ & $.374$ & $.529$\\ 
			$\Motif$ & $\cU_{3}$ &$\cP_3$ & $12$ & $.436$ & $.577$ & $.349$ & $.525$ & $.343$ & $.496$ & $.378$ & $.537$ \\
			\bottomrule
		\end{tabular}
		\label{tab:all results}
	\end{table*}

	\begin{table*}[h!]
		\centering
		\caption{Comparing $\rm{Ultra}, \Motif$ and $\Ultra$ on 5 transductive sparse datasets. Baseline results are taken from \citep{huang2025expressive}.}
		\label{tab:robustness}
		\begin{tabular}{l| cc cc cc cc}
			\toprule
			Model & \multicolumn{2}{c}{$\rm{Ultra}$} & \multicolumn{2}{c}{$\Motif$} & \multicolumn{2}{c}{$\Ultrav{2}$}& \multicolumn{2}{c}{$\Ultrav{3}$} \\
			\midrule
			\textbf{Dataset} & \textbf{MRR} & \textbf{H10} & \textbf{MRR} & \textbf{H10} & \textbf{MRR} & \textbf{H10} & \textbf{MRR} & \textbf{H10}\\
			\midrule
			{WDsinger} & $.382$ & $.498$ & \cellcolor{Goldenrod!25}$\underline{.397}$ & \cellcolor{RoyalBlue!25}$\textbf{.514}$  & \cellcolor{RoyalBlue!25}$\textbf{.402}$ & $.505$  & \cellcolor{RoyalBlue!25}$\textbf{.402}$ & \cellcolor{Goldenrod!25}$\underline{.511}$ \\
			{NELL23k} & $.239$ & $.408$ & $.220$ & $.384$ & \cellcolor{Goldenrod!25}$\underline{.249}$ & \cellcolor{Goldenrod!25}$\underline{.413}$ & \cellcolor{RoyalBlue!25}$\textbf{.250}$ & \cellcolor{RoyalBlue!25}$\textbf{.419}$ \\
			{FB15k237(10)} & \cellcolor{Goldenrod!25}$\underline{.248}$ & $.398$ & $.236$ & $.384$ & \cellcolor{RoyalBlue!25}$\textbf{.249}$ & \cellcolor{RoyalBlue!25}$\textbf{.404}$ & $.245$ & \cellcolor{Goldenrod!25}$\underline{.400}$  \\
			{FB15k237(20)} & \cellcolor{Goldenrod!25}$\underline{.272}$ & \cellcolor{Goldenrod!25}$\underline{.436}$ & $.259$ & $.422$ & \cellcolor{RoyalBlue!25}$\textbf{.274}$ & \cellcolor{RoyalBlue!25}$\textbf{.439}$ & $.268$ & $.431$\\
			{FB15k237(50)} & $.324$ & \cellcolor{Goldenrod!25}$\underline{.526}$ & $.312$ & $.508$ &\cellcolor{RoyalBlue!25}$\textbf{.329}$ & \cellcolor{RoyalBlue!25}$\textbf{.527}$  & \cellcolor{Goldenrod!25}$\underline{.326}$ & $.524$ \\
			\bottomrule
		\end{tabular}
	\end{table*}

	\paragraph{Benchmarks and Pattern-Matched Relation Graphs.} 
	In our experiments, we assess the essence of $\Ultra$ using 57 KGs with various characteristics. We grouped these KGs according to entity learning (18 KGS), graph transfer (23 KGs), 
	and transductive learning (16 KGs) tasks. Additional information about the KGs in each group is included in Appendix \ref{sec:datasets}.
	$\Ultra$ is pretrained on the CoDEx Medium, FB15k237, and WN18RR KGs, and subsequently assessed in the ZSLP tasks on the remaining 51 KGs.

	\cite{galkin2023towards} and \cite{huang2025expressive} employ sparse matrix multiplication of the (multi-relational) adjacency matrix $\A \in \R^{n\times m\times n}$ and matrices representing head-relation pair 
	$\E_h \in \R^{n \times m}$ and tail-relation pair $\E_t\in\R^{n\times m}$. 
	Multiplying $\E_{\rmm{x}}^T$ by $\E_{\rmm{y}}$ results into the adjacency matrix of the $\rmm{x2y}$ meta-relations used in $\rmm{Ultra}.$ 
	This method provides additional information on the number of occurrences of the respective graphlet in the whole dataset. 
	However, this information is left unused, as only the connection information is represented in the relation KG in Ultra and the relation hypergraph in Motif, respectively. 
	As we also distinguish between closed and open path,  computation via the adjacency matrix becomes computationally expensive. Pattern matching, on the other hand, can be used in a highly parallel fashion to compute the relation graph of any KG. 
	To obtain the relation graph we construct a SPARQL ask query for each element in the vocabulary (see Appendix \ref{sec:sparql_queries} for the exact Queries), which can be run on any rdf KG. 
	This method enables the computation of relation graphs based on vocabularies containing arbitrary graphlets. 
	
	\paragraph{Evaluation Protocol.} We follow the best practices for evaluating KGE models by considering the Mean Reciprocal Rank (MRR), and the Hits at n (Hn, n = 1,3,10) metrics. 
	Link prediction consists of finding  the missing entity $?$ in the queries $Q=r(h, ?)$ or $ Q'= r(?, t)$. First, we created a symbolic inverse relation $r',$ which turns queries with a missing head into $Q = r'(t, ?).$ This means that we only look at queries that are in the form of $Q.$ 
	Next, $\Ultra$ scores and ranks the corrupted triples in a decreasing order. 
	The predicted missing entity is the top ranked corrupted entity. 
	We compare our models against the state-of-the-art KGFMs $\rmm{Ultra}$ and $\Motif$ using the aforementioned evaluation metrics.  We consider six different vocabularies to design our models; namely, $\cU_2,$ $\cV^-_j = \cV_j\setminus \set{\rmm{u_{\smile}v_c}},$ $ \cV_j,$ $\cV^+_j = \cV_j\cup \mc{M}'_4, j=2,3$ . We denote the $\Ultra$ variant built on the vocabulary $\mc{X}_{\bullet}^{\pm}$ by $\Ultra[\mc{X}_{\bullet}^{\pm}]$.
	
	\paragraph{Results.} The experimental results of evaluating the pretrained $\Ultra$ on the benchmark KGs are reported in Table~\ref{tab:all results}. In the following, the operator  $\geq$ relating two models means that the first model \emph{outperforms} the second model.

	\begin{figure*}[!ht]
		\begin{center}
			\includegraphics[width=0.9\linewidth]{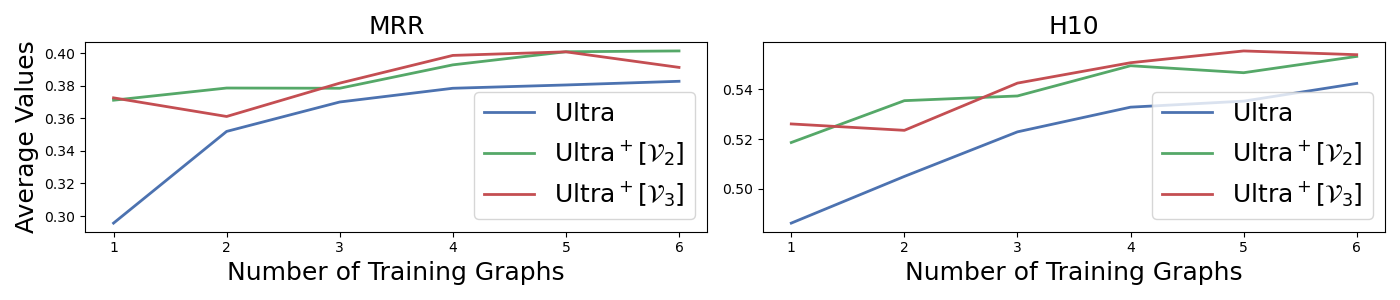}
		\end{center}
		\caption{Average Performance  over 51 Graphs of Ultra and $\Ultra$ models pretrained on an increasing number of Graphs. }
		\label{fig:ultra}
	\end{figure*}
	
	{\it \textbf{General Overview:}} 
	Figure~\ref{fig:ultra} reports the average performance of $\rm{Ultra}$ and $\Ultra$ as the number of pretraining KGs increases. (\textbf{RQ1}) $\Ultra$'s variants consistently outperform $\rm{Ultra}$, demonstrating both scalability and performance gains. While $\Ultrav{2}$ improves monotonically before saturating, $\Ultrav{3}$ fluctuates with the addition of WN18RR and ConceptNet100k at position 2 and 6 respectively (see Table~\ref{tab:pretraining_mixtures} for more details), highlighting the need for careful selection and ordering of pretraining KGs given their structural heterogeneity.  
	Observing the path- and topology- based vocabulary variants, we notice that $\Ultrav{3} $ $\geq$ $\Ultrav{2}$ $\geq$ $\Ultrau{2}$ $\geq$ $\Ultram{2}$ 
	on average for relation learning, graph transfer, and transductive inference tasks. 
	This trend in performance can be related to the inclusion of the vocabularies $\cV^-_2 \subset \cU_2 \subset \cV_2 \subset \cV_3.$
	However, we have found that combining the path-based and topology-based vocabularies does not result in an increase in performance as $\Ultrav{3}$ and $\Ultrav{2}$ often surpass  $\Ultrap{3}$ and $\Ultrap{2}$ respectively. 
	\textbf{(RQ2)} On one hand, we can conclude that the zero-shot performance increases with increasing the path-based vocabulary. 
	\textbf{(RQ3)} On the other hand, mixing the topology-based with the path-based vocabulary does not necessarily preserve the performance increase. 
	
	The $\Motif$ model maintains its superiority over $\rmm{Ultra}$ for all tasks. 
	This  difference in performance is a consequence of adding 3-path graphlets to $\rmm{Ultra}$'s vocabulary. Although $\Ultrav{2}$ utilized only the 2-path based vocabulary $\cV_2$, it notably outperforms $\Motif$ and $\rmm{Ultra}$ on {both} inductive {and transductive} link prediction. \textbf{(RQ4)} This clearly demonstrates the importance of having a vocabulary rich enough to convey information about closed and open paths. 
	$\Ultrav{3}$ is the best performing variant of the $\Ultra$ models across all the settings. It uses the same vocabulary as $\Motif$, except that its graphlets are positional binary orders. \textbf{(RQ5)} Thus, its superiority over $\Motif$ lies in the arity of the graphlets. 
	\textit{ \textbf{ On the robustness of $\Ultra$:}} In all three models: $\Ultra$, Ultra, and Motif, two relations are connected in the relation graph as soon as they co-occur at least once in the KG.
	For $\Ultra$, this corresponds to observing a single match of the associated SPARQL query pattern in the KG; for Ultra and Motif, it corresponds to the relevant entry of the adjacency matrix (obtained via sparse matrix multiplication) becoming non-zero. In either case, the frequency of co-occurrence does not influence whether an edge is created. 
	However, because $\Ultra$ employs positional binary orders, this insensitivity to frequency is implicitly mitigated in its relation graph, as formalized in Theorem~\ref{theo:robustness}. 
	Empirically, sparse KGs provide a natural setting for assessing the robustness of KGFMs that construct relation graphs. In our experiments, WDsinger, NELL23k, and FB15k237(10/20/50) constitute such sparse knowledge graphs. 
	Table ~\ref{tab:robustness} confirms that U$\Ultra$ is superior to the baseline models when it comes to sparse KGs.

	\section{Conclusion}
	We proposed a KGFM framework called $\Ultra$ 
	capable of constructing a relation graph from any structural vocabulary composed of a set of graphlets. 
	This framework enables the conversion of n-ary graphlets' orders into positional binary orders, thereby maintaining pairwise relational semantics and mitigating the complexity associated with hypergraphs. 
	Using SPARQL to run ASK queries simplifies the distinction between open and closed paths when mining graphlets, and overcomes the major limitation of computing relation graphs when higher-order graphlets involve the full adjacency matrix.  
	Our theoretical findings, described in Theorems~\ref{theo:pos order spans n-ary} and ~\ref{theo:robustness}, demonstrate that $\Ultra$ exhibits greater robustness compared to the current baseline KGFMs. Evaluation of ZSLP tasks, with $\Ultra$ pretrained on three KGs, indicated that an increase in structural vocabulary is advantageous when only path-based vocabulary is utilized, yet it becomes detrimental when combining path- and topology-based vocabularies. 
	We showed that enhancing the model's awareness of relational patterns and topological patterns significantly improves the model's MRR and H10, respectively.
	
	
	Our Model $\Ultrav{3}$ achieves state-of-the-art performance in ZSLP averaged across 51 datasets with only 3 Graphs used for pretraining. Our investigation also shows that scaling pretraining has the chance to further improve performance. The case for scaling the vocabulary, on the other hand, remains ambiguous. We observed an increase in performance for increasing path based vocabulary, while adding structurally inspired graphlets seems to be detrimental. A large scale investigation of higher order structural vocabularies remains challenging, due to the computational complexity of relation graph computation for vocabularies containing complex graphlets. We will address efficient computation of relation graphs that go beyond instance based computation (where existence of a single instance of a grpahlet results in a connection in the relation graph) in future research.


	\section*{Impact Statement}
	This paper presents work whose goal is to advance the field of Machine Learning. There are many potential societal consequences of our work, none of which we feel must be specifically highlighted here.
	
	\bibliography{references-add}
	\bibliographystyle{plainnat}
	
	\newpage
	\appendix
	\onecolumn
	\section{Proof of the Theorems\label{appendix:proofs}}
In this section, we provide the proofs for the theorems stated in the paper. Before we begin, let us summarize the key mathematical symbols and their meanings used throughout the paper in Table~\ref{tab:Notation table}
\begin{table}[!h]
\centering
    \caption{Notation Table.}
    \label{tab:Notation table}
    \begin{tabular}{ll}
         \toprule
         Symbol & Meaning\\
         \midrule
        $\mathcal{E}, \mathcal{R}$ & Entity and relation sets\\  
        $\mathcal{T}_{ , +,-}$  &  Set of all plausible, true ($_+$), and corrupted ($_-$) triples\\
        $K=(\mathcal{E}^K, \mathcal{R}^K, \mathcal{T}^K_+)$ & Knowledge Graph \\
        $\phi, \rho, \eta$ & KG, relation, and entity homomorphism \\
        \multirow{2}{*}{$\phi_{\mgg}=(\rho, \eta): G\xrightarrow{mon} K$}& KG monomorphism with \\
        & $\rho: \cR \to \cR^K, \eta:\cE \to \cE^K$ two injective maps \\
        $card(\cG) = \norm{\cT}$ & Cardinality of the graphlet $\cG$\\
        $\cG, G, \mgg $ & Graphlet, small KG, and order on $\cR$ \\
        $\mgg(\cR)$ & A tuple defined by the order $\mgg$ on $\cR$  \\
        $\mgg_{i_1, \ldots, i_{m\leq n}}$ & (if $m<n$) positional m-, (else) n-ary order   \\
        $\mgg(\cR^K) = \set{\mgg\circ\rho (\cR) \vert \rho:\cR \xrightarrow{mon}\cR^K}$ & Ordered n-aries (of relations in $\cR^K$) induced by $\mgg$\\
        $\phi_{\mgg}(G), \cG(K)$ & Occurrence, set of occurrences of $\cG$ in $K$ \\
        $\overline{\rmm{X}} = \set{\phi'(G)\vert \phi'(X)\equiv_{\mgg}\phi(X) }$& Equivalence class  \\
        \multirow{2}{*}{$(\cV, \omega)$}& Structural vocabulary, $\cV =(\cG_i, \rmm{g}_i)$ a set of graphlets, \\
        & $\omega(\overline{\phi(G_i)})$ a weighting function\\
        $\rmm{f, r,\ _o,\_ c}$ & Forward, backward/reversed, open, and closed path \\
        $_{\smile}\in \set{\rmm{f, r}}^{i}$ & A sequence of length $i$ over the alphabets $\rmm{f, r}.$\\
        $u_{\smile} v = \rmm{g}_{1, n}$ & A Positional 2-ary defined by the first and last position\\
        $\cP_p$& Set of paths of length $p$\\
        O., C., O. \& C. $\cP_p$ & Open, closed, and open or closed paths of length $p.$\\
        $\cV_p $ = $ \big\{\rmm{u_{\smile} v_z}\vert \rmm{u, v}\in \set{\rmm{f, r}},$ $\rmm{z\in\set{o, c}}\big\}$ & (O., C.) $\cP_p$-based vocabulary\\
        $\cU_p $ = $ \big\{\rmm{u_{\smile} v}\vert \rmm{u, v}\in \set{\rmm{f, r}},$ $\big\}$ & $\cP_p$-based vocabulary\\
        N-M & Many to many\\
        $\mathcal{M}_{ij} = \set{\rmm{uv}_{i:j}\mid \rmm{u, v} \in \set{\rmm{f, r}}}$ & N-M subgraphs with $i, j$ number of $\rmm{u}$ and $\rmm{v}$ edges resp.\\
        $\mathcal{M}_{m} = \set{\mathcal{M}_{i, j}\mid i + j = m}$ & m-star\\
        $\cV_{p}^{-}$ & O. $\cP_{p}$-based vocabulary\\
        $\cV_p^{+}$ & O. \& C. $\cP_p$- and $\mathcal{M}_4$-based vocabulary\\
        $\mc{N}^{i}(\rmm{g}, r)$ & Neighborhood of $r$ (the $i$th node) relative to the n-ary $\rmm{g}$\\
        $\Ultra[\mc{X}_{p}^{\pm}]$ & $\Ultra$ variant built on the vocabulary $\mc{X}_{p}^{\pm}$
    \end{tabular}
\end{table}

\subsection{Proof of Theorem~\ref{theo:pos order spans n-ary}\label{appendix:theo-pos order spans n-ary}}
Any positional m-ary order, $m < n,$ spans a group of n-ary orders.  

\begin{proof}
    Consider two positive integers, $m$ and $n$, where $m < n$. Let $\mgg^{} = \mgg_{i_1, \ldots, i_m \leq n}$ represents a positional m-ary. According to Definition~\ref{def:graphlet}, $\mgg$ is a function with $n$ arguments, out of which $n-m$ are dummy arguments. Define $\pi(i_k)$ as the position of $i_k$ within the ordered list of all $n$ arguments. We now consider the family of n-ary $\mgg^{(j)} = \mgg_{j_1, \ldots, j_n}$ such that if $\pi(j_l) = \pi(i_k)$, then $j_l = i_k$ for $k=1, \ldots, m$ and $l=1, \ldots, n$. Consequently, $\mgg$ is induced by any of the n-ary $\mgg^{(j)}$. In other words, the m-ary spans the aforementioned family of n-aries.
\end{proof}

\subsection{Proof of Theorem~\ref{theo:robustness}\label{appendix:theo-robustness}}
Let $\rho$ be a monomorphism from $\cP_3$ to a graph $K.$ If $\rmm{uvw_o\circ \rho}$ is an $\epsilon$-edge and its corresponding motif in $\Motif$'s vocabulary is an $\epsilon'$-edge, then $\epsilon'\leq \epsilon.$

\begin{proof}
    Let $\rho$ be a monomorphism from $\cP_3$ to a graph $K.$ $\rmm{uvw_o}$ is a positional binary order whose second argument is the dummy argument; i.e. $\rmm{uvw_o} = \mgg_{1,3}.$
    It follows from Theorem~\ref{theo:pos order spans n-ary} that $\rmm{uvw_o}$ spans a group of 3-ary, $\mgg^{(j)},$ including it corresponding motif $\mu.$ 
    All triples in the equivalence classes $\mgg^{(j)}\circ \rho(\cR)$ are also in $\rmm{uvw_o\circ \rho(\cR)},$ so that $\epsilon $ = $ \omega\left(\rmm{uvw_o}\circ \rho(\cR)\right) $ = $ \sum_j\omega\left(\mgg^{(j)}\circ \rho(\cR)\right) $ $\geq$ $ \mu\circ \rho(\cR) $=$ \epsilon'$
\end{proof}

\subsection{Expressiveness Limitation of Motif Augmented with a Closed Path Compare to $\Ultra$}

\definecolor{darkyellow}{RGB}{204,153,0}
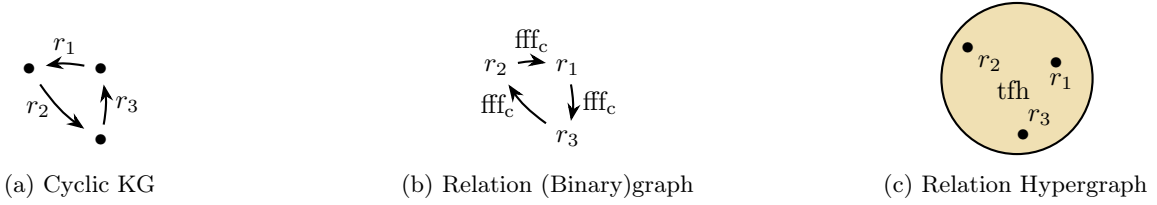
\begin{figure}
\centering
\begin{subfigure}{0.25\textwidth}
\centering
\begin{tikzpicture}[
  node distance=.95cm,
  every node/.style={},
  ->, >=Stealth, thick,
  bend angle=10
]
\node (e1) {$\bullet$};
\node[left of=e1] (e2) {$\bullet$};
\node[below of=e1] (e3) {$\bullet$};

\draw (e1) edge[bend right] node[above] {$r_1$} (e2);
\draw (e2) edge[bend right] node[below, left] {$r_2$} (e3);
\draw (e3) edge[bend right] node[below, right] {$r_3$} (e1);
\end{tikzpicture}
\caption{Cyclic KG}
\label{fig:Cycle-graph G}
\end{subfigure}
\hfill
\begin{subfigure}{0.25\textwidth}
\centering
\begin{tikzpicture}[
  node distance=.95cm,
  every node/.style={},
  ->, >=Stealth, thick,
  bend angle=10
]
\node (r1) {$r_1$};
\node[left of=r1] (r2) {$r_2$};
\node[below of=r1] (r3) {$r_3$};

\draw (r2) edge[bend left] node[above] {$\rmm{fff_c}$} (r1);
\draw (r3) edge[bend left] node[below, left] {$\rmm{fff_c}$} (r2);
\draw (r1) edge[bend left] node[below, right] {$\rmm{fff_c}$} (r3);
\end{tikzpicture}
\caption{ Relation (Binary)graph}
\label{fig:Cycle-graph Ultra}
\end{subfigure}
\hfill
\begin{subfigure}{0.25\textwidth}
\centering
\begin{tikzpicture}[
  node distance=.95cm,
  every node/.style={},
  ->, >=Stealth, thick,
  bend angle=10
]
\node (r11) {$r_1$};
\node[left of=r11] (r12) {$r_2$};
\node[below of=r11] (r13) {$r_3$};

\draw[thick, fill=darkyellow!30] (-0.3,-0.2) circle (1cm) node[above right, yshift=-0.4cm, xshift=-0.4cm] {$\rmm{tfh}$};
\draw (r12) node[above] {$\bullet$} (r11) node[below right] {$r_1$} (r11);
\draw (r13) node[below, left] {$\bullet$} (r12) node[above, right] {$r_2$} (r12);
\draw (r11) node[below, right] {$\bullet$} (r13) node[above] {$r_3$} (r13);
\end{tikzpicture}
\caption{Relation Hypergraph}
\label{fig:Cycle-graph Motif}
\end{subfigure}
\caption{Cyclic Knowledge Graph and Relation Graphs: (a) A cyclic knowledge graph with three relations. (b) $\Ultra $ constructs a relation graph consisting of three 2-ary edges, while (c) Motif constructs a relation hypergraph with a single 3-ary edge.}
\label{fig:Cycle-graph}
\end{figure}
One of $\Ultra$'s contributions is distinguishing between closed and open paths. Although the GNN architectures of $\Ultra$ and Motif are different, we are interested in finding out whether Motif augmented with closed paths is more expressive than $\Ultra.$ 
Let us assume that Motif is more expressive than $\Ultra.$ 
We will now consider the cyclic KG in Figure~\ref{fig:Cycle-graph} and its relation graph in the $\Ultra$ and Motif framework. As this graph is symmetric, the results are independent of the choice of query relation. 
Let $r_1$ be the query relation. 
\paragraph{Relation Encoding with Motif.}
The relation graph, $G_M,$ constructed in Motif framework has three edges, mainly: $\rmm{tfh}(r_1, r_{2}, r_{3})$, $\rmm{tfh}(r_3, r_{1}, r_{2})$, and $\rmm{tfh}(r_2, r_{3}, r_{1}).$ We have
$\Enc{r_1}{r_1}{0} = \mathbf{1}$ and $\Enc{r_i}{r_1}{0} = \mathbf{0}, i\neq 1.$ Thus,
\begin{align*}
    \Enc{r_1}{r_1}{1} &= \rmm{UP}\big(\Enc{r_1}{r_1}{0}, \rmm{AGG}\big[ \big\{\rmm{MSG}\big(\big\{\big(\Enc{r'}{r_1}{0}, \mf{tfh}\big) 
    \vert r' \in \{ r_2, r_3\}\big\}\big)\big\}\big]\big)\\
    &= \mathbf{1};\\
    \Enc{r_2}{r_1}{1} &= \rmm{UP}\big(\Enc{r_2}{r_1}{0}, \rmm{AGG}\big[ \big\{\rmm{MSG}\big(\big\{\big(\Enc{r'}{r_1}{0}, \mf{tfh}\big) 
    \vert r' \in \{ r_1, r_3\}\big\}\big)\big\}\big]\big)\\
    &= \mathbf{1}\times \mf{tfh} = \mf{tfh};\\
    \Enc{r_3}{r_1}{1} &= \rmm{UP}\big(\Enc{r_3}{r_1}{0}, \rmm{AGG}\big[ \big\{\rmm{MSG}\big(\big\{\big(\Enc{r'}{r_1}{0}, \mf{tfh}\big) 
    \vert r' \in \{ r_1, r_2\}\big\}\big)\big\}\big]\big)\\
    &= \mathbf{1}\times \mf{tfh} = \mf{tfh};
\end{align*}
We observe that $\Enc{r_2}{r_1}{1} = \Enc{r_3}{r_1}{1}$ and assume that this holds for all $t\leq T$ where $T>1.$ It follows that 
\begin{align*}
    \Enc{r_2}{r_1}{T+1} &= \rmm{UP}\big(\Enc{r_2}{r_1}{T}, \rmm{AGG}\big[ \big\{\rmm{MSG}\big(\big\{\big(\Enc{r'}{r_1}{T}, \mf{tfh}\big) 
    \vert r' \in \{ r_1, r_3\}\big\}\big)\big\}\big]\big)\\
    &= \Enc{r_2}{r_1}{T} + (\Enc{r_1}{r_1}{T} + \Enc{r_3}{r_1}{T})\times \mf{tfh};\\
    \Enc{r_3}{r_1}{T+1} &= \rmm{UP}\big(\Enc{r_3}{r_1}{T}, \rmm{AGG}\big[ \big\{\rmm{MSG}\big(\big\{\big(\Enc{r'}{r_1}{T}, \mf{tfh}\big) 
    \vert r' \in \{ r_1, r_2\}\big\}\big)\big\}\big]\big)\\
    &=\Enc{r_3}{r_1}{T} +  (\Enc{r_1}{r_1}{T} + \Enc{r_2}{r_1}{T})\times \mf{tfh}\\
    &= \Enc{r_2}{r_1}{T} + (\Enc{r_1}{r_1}{T} + \Enc{r_3}{r_1}{T})\times \mf{tfh}\\
    &= \Enc{r_2}{r_1}{T+1}.
\end{align*}
Thus, Motif can't differentiate $r_2$ from $r_3$. 

\paragraph{Relation Encoding with $\Ultra$.} On the other hand, from $\Enc{r_1}{r_1}{0} = \mathbf{1}$ and $\Enc{r_i}{r_1}{0} = \mathbf{0}, i\neq 1,$ $\Ultra$'s relation embedding yields 
\begin{align*}
    \Enc{r_1}{r_1}{1} &= \rmm{UP}\big(\Enc{r_1}{r_1}{0}, \rmm{AGG}\big[ \big\{\rmm{MSG}\big(\big\{\big(\Enc{r'}{r_1}{0}, \mf{fff_c}\big) 
    \vert r' \in \{ r_2\}\big\}\big)\big\}\big]\big)\\
    &= \mathbf{1};\\
    \Enc{r_2}{r_1}{1} &= \rmm{UP}\big(\Enc{r_2}{r_1}{0}, \rmm{AGG}\big[ \big\{\rmm{MSG}\big(\big\{\big(\Enc{r'}{r_1}{0}, \mf{fff_c}\big) 
    \vert r' \in \{r_3\}\big\}\big)\big\}\big]\big)\\
    &= \mathbf{0};\\
    \Enc{r_3}{r_1}{1} &= \rmm{UP}\big(\Enc{r_3}{r_1}{0}, \rmm{AGG}\big[ \big\{\rmm{MSG}\big(\big\{\big(\Enc{r'}{r_1}{0}, \mf{fff_c}\big) 
    \vert r' \in \{ r_1\}\big\}\big)\big\}\big]\big)\\
    &= \mathbf{1}\times \mf{fff_c} = \mf{fff_c}.
\end{align*}
Since  $\Enc{r_{2}}{r_1}{1} \neq \Enc{r_{3}}{r_1}{1},$ let us assume that $\Enc{r_{2,3}}{r_1}{t}$ are distinct vectors for $t < T$ except for $\Enc{r_2}{r_1}{T} = \Enc{r_3}{r_1}{T} = \mathbf{c}_T.$ 
We would then have
\begin{align*}
    \Enc{r_1}{r_1}{T+1} &= \rmm{UP}\big(\Enc{r_1}{r_1}{T}, \rmm{AGG}\big[ \big\{\rmm{MSG}\big(\big\{\big(\Enc{r'}{r_1}{T}, \mf{fff_c}\big) 
    \vert r' \in \{ r_2\}\big\}\big)\big\}\big]\big)\\
    &= \Enc{r_1}{r_1}{T} + \mf{c}_T\mf{\times fff_c};\\
    \Enc{r_2}{r_1}{T+1} &= \rmm{UP}\big(\Enc{r_2}{r_1}{T}, \rmm{AGG}\big[ \big\{\rmm{MSG}\big(\big\{\big(\Enc{r'}{r_1}{T}, \mf{fff_c}\big) 
    \vert r' \in \{r_3\}\big\}\big)\big\}\big]\big)\\
    &= \mf{c}_T  + \mf{c}_T\mf{\times fff_c} ;\\
    \Enc{r_3}{r_1}{T+1} &= \rmm{UP}\big(\Enc{r_3}{r_1}{T}, \rmm{AGG}\big[ \big\{\rmm{MSG}\big(\big\{\big(\Enc{r'}{r_1}{T}, \mf{fff_c}\big) 
    \vert r' \in \{ r_1\}\big\}\big)\big\}\big]\big)\\
    &= \mathbf{c}_T + \Enc{r_1}{r_1}{T}\times \mf{fff_c}\\
    &\neq \Enc{r_2}{r_1}{T+1}
\end{align*}
since $\Enc{r_1}{r_1}{t}$ and $\mf{c}_t$ are multivalued polynomials of indeterminate $\mf{fff_c}$ and constant terms $\mf{1}$ and $\mf{0}$ respectively. In other words, $\Ultra$ is able to distinguish between $r_2$ and $r_3$.
This contradicts our assumption about the expressive power of Motif. 

We can then conclude that $\Ultra$ is at least as expressive as  Motif. 
	
	\section{Datasets}\label{sec:datasets}

Our Experiments have been performed on a multitude of datasets, following \citep{galkin2023towards}. These datasets can be grouped into the following  three subsets:

\begin{itemize}
    \item \textbf{Inductive $(e,r)$}: Inductive link prediction datasets with prediction on new nodes and new relations.
    \item \textbf{Inductive $(e)$}: Datasets for inductive Link prediction on new nodes.
    \item \textbf{Transductive}: Transductive Link prediction on seen nodes and relations.
\end{itemize}

Datasets and corresponding statistics are displayed in tables~\ref{tab:inductive-e-r-statistics}-\ref{tab:transductive-statistics}.

\begin{table*}[t]
    \centering
    \small
    \caption{Statistics of \textbf{inductive} $(e,r)$ link prediction datasets. Triples are the
number of edges given at training, validation, or test graphs, respectively, whereas Valid and Test denote triples to be predicted in the validation and test graphs.}
    \label{tab:inductive-e-r-statistics}
    \begin{tabular}{lccc|cccc|ccccc}
    \toprule
 \multirow{2}{*}{ \textbf{Dataset }} & \multicolumn{3}{c}{ \textbf{Training Graph} } & \multicolumn{4}{c}{ \textbf{Validation Graph} } & \multicolumn{4}{c}{ \textbf{Test Graph} }  \\
 \cmidrule{2-12}
 & \textbf{Entities} & \textbf{Rels} & \textbf{Triples} & \textbf{Entities} & \textbf{Rels} & \textbf{Triples} & \textbf{Valid} & \textbf{Entities} & \textbf{Rels} & \textbf{Triples} & \textbf{Test}  \\
 \midrule
 FB-25  & 5190 & 163 & 91571 & 4097 & 216 & 17147 & 5716 & 4097 & 216 & 17147 & 5716  \\
 FB-50  & 5190 & 153 & 85375 & 4445 & 205 & 11636 & 3879 & 4445 & 205 & 11636 & 3879  \\
 FB-75  & 4659 & 134 & 62809 & 2792 & 186 & 9316 & 3106 & 2792 & 186 & 9316 & 3106  \\
 FB-100  & 4659 & 134 & 62809 & 2624 & 77 & 6987 & 2329 & 2624 & 77 & 6987 & 2329  \\
 WK-25  & 12659 & 47 & 41873 & 3228 & 74 & 3391 & 1130 & 3228 & 74 & 3391 & 1131  \\
 WK-50  & 12022 & 72 & 82481 & 9328 & 93 & 9672 & 3224 & 9328 & 93 & 9672 & 3225  \\
 WK-75  & 6853 & 52 & 28741 & 2722 & 65 & 3430 & 1143 & 2722 & 65 & 3430 & 1144  \\
 WK-100  & 9784 & 67 & 49875 & 12136 & 37 & 13487 & 4496 & 12136 & 37 & 13487 & 4496  \\
 NL-0  & 1814 & 134 & 7796 & 2026 & 112 & 2287 & 763 & 2026 & 112 & 2287 & 763  \\
 NL-25  & 4396 & 106 & 17578 & 2146 & 120 & 2230 & 743 & 2146 & 120 & 2230 & 744  \\
 NL-50  & 4396 & 106 & 17578 & 2335 & 119 & 2576 & 859 & 2335 & 119 & 2576 & 859  \\
 NL-75  & 2607 & 96 & 11058 & 1578 & 116 & 1818 & 606 & 1578 & 116 & 1818 & 607  \\
 NL-100  & 1258 & 55 & 7832 & 1709 & 53 & 2378 & 793 & 1709 & 53 & 2378 & 793  \\
 Metafam  & 1316 & 28 & 13821 & 1316 & 28 & 13821 & 590 & 656 & 28 & 7257 & 184   \\
 FBNELL  & 4636 & 100 & 10275 & 4636 & 100 & 10275 & 1055 & 4752 & 183 & 10685 & 597   \\
 \midrule
 Wiki MT1 tax  & 10000 & 10 & 17178 & 10000 & 10 & 17178 & 1908 & 10000 & 9 & 16526 & 1834   \\
 Wiki MT1 health  & 10000 & 7 & 14371 & 10000 & 7 & 14371 & 1596 & 10000 & 7 & 14110 & 1566   \\
 Wiki MT2 org  & 10000 & 10 & 23233 & 10000 & 10 & 23233 & 2581 & 10000 & 11 & 21976 & 2441  \\
 Wiki MT2 sci  & 10000 & 16 & 16471 & 10000 & 16 & 16471 & 1830 & 10000 & 16 & 14852 & 1650  \\
 Wiki MT3 art  & 10000 & 45 & 27262 & 10000 & 45 & 27262 & 3026 & 10000 & 45 & 28023 & 3113  \\
 Wiki MT3 infra  & 10000 & 24 & 21990 & 10000 & 24 & 21990 & 2443 & 10000 & 27 & 21646 & 2405  \\
 Wiki MT4 sci  & 10000 & 42 & 12576 & 10000 & 42 & 12576 & 1397 & 10000 & 42 & 12516 & 1388  \\
 Wiki MT4 health  & 10000 & 21 & 15539 & 10000 & 21 & 15539 & 1725 & 10000 & 20 & 15337 & 1703  \\
\bottomrule
\end{tabular}
\end{table*}

\begin{table*}[t]
    \small
    \centering
     \caption{Statistics of inductive $(e)$ link prediction datasets. Triples are the
number of edges given at training, validation, or test graphs, respectively, whereas Valid and Test denote triples to be predicted in the validation and test graphs.}
\label{tab:inductive-e-statistics}
    \begin{tabular}{lccc|ccc|ccc}
\toprule \multirow{2}{*}{ \textbf{Dataset }} & \multirow{2}{*}{ \textbf{Rels} } & \multicolumn{2}{c}{ \textbf{Training Graph} } & \multicolumn{3}{c}{ \textbf{Validation Graph} } & \multicolumn{3}{c}{ \textbf{Test Graph} }  \\
\cmidrule{3-10} & & \textbf{Entities} & \textbf{Triples} & \textbf{Entities} & \textbf{Triples} & \textbf{Valid} & \textbf{Entities} & \textbf{Triples} & \textbf{Test} \\
\midrule
 FB-v1  & 180 & 1594 & 4245 & 1594 & 4245 & 489 & 1093 & 1993 & 411 \\
 FB-v2 & 200 & 2608 & 9739 & 2608 & 9739 & 1166 & 1660 & 4145 & 947  \\
 FB-v3  & 215 & 3668 & 17986 & 3668 & 17986 & 2194 & 2501 & 7406 & 1731  \\
 FB-v4  & 219 & 4707 & 27203 & 4707 & 27203 & 3352 & 3051 & 11714 & 2840 \\
 WN-v1  & 9 & 2746 & 5410 & 2746 & 5410 & 630 & 922 & 1618 & 373  \\
 WN-v2  & 10 & 6954 & 15262 & 6954 & 15262 & 1838 & 2757 & 4011 & 852  \\
 WN-v3  & 11 & 12078 & 25901 & 12078 & 25901 & 3097 & 5084 & 6327 & 1143  \\
 WN-v4  & 9 & 3861 & 7940 & 3861 & 7940 & 934 & 7084 & 12334 & 2823  \\
 NL-v1  & 14 & 3103 & 4687 & 3103 & 4687 & 414 & 225 & 833 & 201  \\
 NL-v2  & 88 & 2564 & 8219 & 2564 & 8219 & 922 & 2086 & 4586 & 935  \\
 NL-v3  & 142 & 4647 & 16393 & 4647 & 16393 & 1851 & 3566 & 8048 & 1620  \\
 NL-v4  & 76 & 2092 & 7546 & 2092 & 7546 & 876 & 2795 & 7073 & 1447  \\
 ILPC Small  & 48 & 10230 & 78616 & 6653 & 20960 & 2908 & 6653 & 20960 & 2902  \\
 ILPC Large  & 65 & 46626 & 202446 & 29246 & 77044 & 10179 & 29246 & 77044 & 10184 \\
 HM 1k & 11 & 36237 & 93364 & 36311 & 93364 & 1771 & 9899 & 18638 & 476 \\
 HM 3k  & 11 & 32118 & 71097 & 32250 & 71097 & 1201 & 19218 & 38285 & 1349  \\
 HM 5k & 11 & 28601 & 57601 & 28744 & 57601 & 900 & 23792 & 48425 & 2124  \\
 HM Indigo& 229 & 12721 & 121601 & 12797 & 121601 & 14121 & 14775 & 250195 & 14904  \\
\bottomrule
\end{tabular}
\end{table*}

\begin{table*}[t]
    \centering
    
    \caption{Statistics of transductive link prediction datasets. Task denotes the prediction task: $h/t$ is predicting both heads and tails, and $t$ is
predicting only tails. }
    \label{tab:transductive-statistics}
\begin{tabular}{lccccccc}
\toprule \textbf{Dataset} & \textbf{Entities} & \textbf{Rels} & \textbf{Train} & \textbf{Valid} & \textbf{Test} & \textbf{Task} \\
\midrule
FB15k237  & 14541 & 237 & 272115 & 17535 & 20466 & $h/t$ \\
WN18RR  & 40943 & 11 & 86835 & 3034 & 3134 & $h/t$  \\
 CoDEx Small  & 2034 & 42 & 32888 & 1827 & 1828 & $h/t$ \\
CoDEx Medium  & 17050 & 51 & 185584 & 10310 & 10311 & $h/t$  \\
CoDEx Large & 77951 & 69 & 551193 & 30622 & 30622 & $h/t$ \\
NELL995 & 74536 & 200 & 149678 & 543 & 2818 & $h/t$ \\
YAGO310 & 123182 & 37 & 1079040 & 5000 & 5000 & $h/t$  \\
WDsinger  & 10282 & 135 & 16142 & 2163 & 2203 & $h/t$  \\
NELL23k  & 22925 & 200 & 25445 & 4961 & 4952 & $h/t$ \\
FB15k237(10)  & 11512 & 237 & 27211 & 15624 & 18150 & $t$  \\
FB15k237(20) & 13166 & 237 & 54423 & 16963 & 19776 &$t$ \\
FB15k237(50)  & 14149 & 237 & 136057 & 17449 & 20324 & $t$ \\
Hetionet & 45158 & 24 & 2025177 & 112510 & 112510 & $h/t$  \\
ConceptNet100k & 78334 & 34 & 100000 & 1200 & 1200 & $h/t$\\
\bottomrule
\end{tabular}
\end{table*}

	\section{Additional Results}\label{sec:additional-results}
\tabcolsep=0.1cm
\begin{table*}[h!]
\small
\centering
\caption{Full results for $\rm{Ultra}$ and $\Ultra$ models on 10 transductive datasets. Baseline results are taken from \citep{huang2025expressive}.}
\label{tab:model_performance_transductive}
\begin{tabular}{l|cc|cc cc cc cc cc cc}
\toprule
Model & \multicolumn{2}{|c|}{$\rm{Ultra}$} & \multicolumn{12}{c}{$\Ultra$} \\
\cmidrule{2-3} \cmidrule{3-15}
Vocab. & \multicolumn{2}{|c|}{$\cU$} & \multicolumn{2}{c}{$\cV_2^-$} & \multicolumn{2}{c}{$\cV_2$} & \multicolumn{2}{c}{$\cV_2^+$}& \multicolumn{2}{c}{$\cV_3^-$} & \multicolumn{2}{c}{$\cV_3$} & \multicolumn{2}{c}{$\cV_3^+$} \\
\midrule
 \textbf{Dataset} & \textbf{MRR} & \textbf{H10} & \textbf{MRR} & \textbf{H10} & \textbf{MRR} & \textbf{H10} & \textbf{MRR} & \textbf{H10} & \textbf{MRR} & \textbf{H10} & \textbf{MRR} & \textbf{H10} & \textbf{MRR} & \textbf{H10} \\
\midrule
{CoDExSmall} & $.479$ & $.668$ & $.469$ & $\textbf{.675}$ & $\textbf{.486}$ & $.675$ & $.480$ & $.675$ & $.475$ & $.667$ & $.484$ & $.674$ & $.475$ & $.671$ \\
{CoDExLarge} & $.339$ & $.466$ & $.343$ & $.470$ & $.342$ & $.471$ & $.336$ & $.466$ & $\textbf{.343}$ & $\textbf{.473}$ & $.340$ & $.468$ & $.335$ & $.462$ \\
{NELL995} & $.444$ & $.583$ & $.461$ & $.598$ & $.446$ & $.600$ & $\textbf{.507}$ & $\textbf{.644}$ & $.472$ & $.601$ & $.451$ & $.613$ & $.484$ & $.637$ \\
{YAGO310} & $.438$ & $.604$ & $.395$ & $.570$ & $.416$ & $.639$ & $.420$ & $.607$ & $.473$ & $.649$ & $\textbf{.505}$ & $\textbf{.669}$ & $.411$ & $.587$ \\
{WDsinger} & $.388$ & $.495$ & $.363$ & $.500$ & $.402$ & $.505$ & $.392$ & $\textbf{.512}$ & $.388$ & $.503$ & $\textbf{.402}$ & $.511$ & $.401$ & $.509$ \\
{NELL23k} & $.228$ & $.392$ & $.224$ & $.392$ & $.249$ & $.413$ & $.238$ & $.405$ & $.224$ & $.388$ & $\textbf{.250}$ & $\textbf{.419}$ & $.241$ & $.401$ \\
{FB15k237(10)} & $.237$ & $.403$ & $.245$ & $.400$ & $\textbf{.249}$ & $\textbf{.404}$ & $.244$ & $.395$ & $.233$ & $.382$ & $.245$ & $.400$ & $.240$ & $.390$ \\
{FB15k237(20)} & $.268$ & $.436$ & $.271$ & $.438$ & $\textbf{.274}$ & $\textbf{.439}$ & $.270$ & $.433$ & $.265$ & $.425$ & $.268$ & $.431$ & $.238$ & $.398$ \\
{FB15k237(50)} & $.323$ & $.525$ & $.325$ & $.525$ & $\textbf{.329}$ & $\textbf{.527}$ & $.324$ & $.526$ & $.323$ & $.519$ & $.326$ & $.524$ & $.313$ & $.502$ \\
{Hetionet} & $.287$ & $\textbf{.417}$& $.282$ & .410& $\textbf{.301}$ & $.417$ & $.259$ & $.381$ &&& $.278$ & $.405$ & $.280$ & $.390$ \\
\bottomrule
\end{tabular}
\end{table*}

\begin{table*}[h!]
\small
\centering
\caption{Full results for $\rm{Ultra}$ and $\Ultra$ models on 23 inductive $(e,r)$ datasets. Baseline results are taken from \citep{huang2025expressive}.}
\label{tab:model_performance_inductive_er}
\begin{tabular}{l|cc|cc cc cc cc cc cc}
\toprule
Model & \multicolumn{2}{|c|}{$\rm{Ultra}$} & \multicolumn{12}{c}{$\Ultra$} \\
\cmidrule{2-3} \cmidrule{3-15}
Vocab. & \multicolumn{2}{|c|}{$\cU$} & \multicolumn{2}{c}{$\cV_2^-$} & \multicolumn{2}{c}{$\cV_2$} & \multicolumn{2}{c}{$\cV_2^+$}& \multicolumn{2}{c}{$\cV_3^-$} & \multicolumn{2}{c}{$\cV_3$} & \multicolumn{2}{c}{$\cV_3^+$} \\
\midrule
 \textbf{Dataset} & \textbf{MRR} & \textbf{H10} & \textbf{MRR} & \textbf{H10} & \textbf{MRR} & \textbf{H10} & \textbf{MRR} & \textbf{H10} & \textbf{MRR} & \textbf{H10} & \textbf{MRR} & \textbf{H10}  & \textbf{MRR} & \textbf{H10}\\
\midrule
FB-25 & $.385$ & $.636$ & $.386$ & $.639$ & $\textbf{.396}$ & $.639$ & $.394$ & $\textbf{.647}$ & $.384$ & $.638$ & $.393$ & $.643$ & $.391$ & $.645$ \\
FB-50 & $.332$ & $.535$ & $.329$ & $.540$ & $.339$ & $.548$ & $.339$ & $\textbf{.551}$ & $.330$ & $.543$ & $.341$ & $.546$ & $\textbf{.343}$ & $.547$ \\
FB-75 & $.397$ & $.596$ & $\textbf{.404}$ & $\textbf{.609}$ & $.404$ & $.605$ & $.403$ & $.607$ & $.399$ & $.604$ & $.404$ & $.605$ & $.398$ & $.603$ \\
FB-100 & $.443$ & $.626$ & $.435$ & $.627$ & $\textbf{.443}$ & $.625$ & $.439$ & $.633$ & $.438$ & $.628$ & $.443$ & $\textbf{.641}$ & $.431$ & $.638$ \\
WK-25 & $.301$ & $.505$ & $.284$ & $.488$ & $\textbf{.324}$ & $\textbf{.530}$ & $.280$ & $.486$ & $.293$ & $.505$ & $.304$ & $.491$ & $.305$ & $.503$ \\
WK-50 & $.157$ & $.305$ & $.130$ & $.287$ & $\textbf{.174}$ & $\textbf{.321}$ & $.168$ & $.315$ & $.159$ & $.285$ & $.168$ & $.319$ & $.166$ & $.307$ \\
WK-75 & $.375$ & $\textbf{.538}$ & $.373$ & $.517$ & $\textbf{.380}$ & $.537$ & $.367$ & $.516$ & $.364$ & $.533$ & $.371$ & $.533$ & $.374$ & $.524$ \\
WK-100 & $.180$ & $.298$ & $.169$ & $.294$ & $.180$ & $.302$ & $.175$ & $.294$ & $.176$ & $.291$ & $.173$ & $.291$ & $\textbf{.185}$ & $\textbf{.307}$ \\
NL-0 & $.334$ & $.510$ & $.318$ & $.502$ & $.367$ & $.551$ & $.336$ & $.525$ & $.303$ & $.505$ & $\textbf{.370}$ & $\textbf{.566}$ & $.364$ & $.546$ \\
NL-25 & $.373$ & $.544$ & $.313$ & $.495$ & $.370$ & $.552$ & $.349$ & $.507$ & $.358$ & $.542$ & $.349$ & $\textbf{.585}$ & $\textbf{.391}$ & $.573$ \\
NL-50 & $.389$ & $.536$ & $.358$ & $.531$ & $\textbf{.406}$ & $\textbf{.579}$ & $.349$ & $.538$ & $.364$ & $.554$ & $.382$ & $.563$ & $.393$ & $.568$ \\
NL-75 & $.336$ & $.528$ & $.307$ & $.487$ & $.348$ & $.529$ & $.302$ & $.495$ & $.316$ & $.490$ & $.348$ & $.539$ & $\textbf{.349}$ & $\textbf{.546}$ \\
NL-100 & $.442$ & $.636$ & $.401$ & $.627$ & $.477$ & $.681$ & $.449$ & $.660$ & $.444$ & $.657$ & $.479$ & $\textbf{.694}$ & $\textbf{.483}$ & $.690$ \\
Metafam & $.428$ & $.739$ & $.156$ & $.503$ & $.262$ & $.723$ & $\textbf{.484}$ & $\textbf{.962}$ & $.173$ & $.560$ & $.279$ & $.851$ & $.310$ & $.872$ \\
FBNELL & $.461$ & $.631$ & $.463$ & $.634$ & $.484$ & $.659$ & $.482$ & $.652$ & $.471$ & $.640$ & $.492$ & $.647$ & $\textbf{.496}$ & $\textbf{.679}$ \\
Wiki MT1 tax & $.240$ & $.306$ & $.150$ & $.300$ & $.260$ & $\textbf{.436}$ & $.251$ & $.311$ & $.234$ & $.347$ & $\textbf{.286}$ & $.433$ & $.238$ & $.305$ \\
Wiki MT1 health & $.327$ & $.430$ & $.291$ & $.394$ & $.362$ & $.432$ & $.312$ & $.400$ & $.373$ & $.457$ & $\textbf{.375}$ & $\textbf{.458}$ & $.363$ & $.449$ \\
Wiki MT2 org & $.089$ & $.152$ & $.096$ & $.157$ & $\textbf{.098}$ & $.158$ & $.091$ & $.156$ & $.096$ & $.159$ & $.098$ & $\textbf{.163}$ & $.096$ & $.159$ \\
Wiki MT2 sci & $.263$ & $.415$ & $.262$ & $.387$ & $.283$ & $.450$ & $.283$ & $.424$ & $.266$ & $.427$ & $\textbf{.300}$ & $\textbf{.458}$ & $.270$ & $.380$ \\
Wiki MT3 art & $.262$ & $.413$ & $.272$ & $.420$ & $.276$ & $.422$ & $.277$ & $.429$ & $.277$ & $.429$ & $\textbf{.286}$ & $\textbf{.435}$ & $.278$ & $.420$ \\
Wiki MT3 infra & $.634$ & $.769$ & $.647$ & $\textbf{.791}$ & $.637$ & $.783$ & $.640$ & $.774$ & $.624$ & $.755$ & $\textbf{.647}$ & $.782$ & $.638$ & $.765$ \\
Wiki MT4 sci & $.285$ & $.449$ & $.295$ & $.469$ & $.301$ & $.464$ & $.294$ & $.466$ & $\textbf{.301}$ & $.465$ & $.295$ & $\textbf{.471}$ & $.296$ & $.463$ \\
Wiki MT4 health & $\textbf{.625}$ & $\textbf{.755}$ & $.595$ & $.746$ & $.568$ & $.729$ & $.558$ & $.723$ & $.598$ & $.729$ & $.583$ & $.744$ & $.619$ & $.746$ \\
\bottomrule
\end{tabular}
\end{table*}

\begin{table*}[h!]
\small
\centering
\caption{Full results for $\rm{Ultra}$ and $\Ultra$ models on 18 inductive $(e)$ datasets. Baseline results are taken from \citep{huang2025expressive}.}
\label{tab:model_performance_inductive_e}
\begin{tabular}{l|cc|cc cc cc cc cc cc}
\toprule
Model & \multicolumn{2}{|c|}{$\rm{Ultra}$} & \multicolumn{12}{c}{$\Ultra$} \\
\cmidrule{2-3} \cmidrule{3-15}
Vocab. & \multicolumn{2}{|c|}{$\cU$} & \multicolumn{2}{c}{$\cV_2^-$} & \multicolumn{2}{c}{$\cV_2$} & \multicolumn{2}{c}{$\cV_2^+$}& \multicolumn{2}{c}{$\cV_3^-$} & \multicolumn{2}{c}{$\cV_3$} & \multicolumn{2}{c}{$\cV_3^+$} \\
\midrule
 \textbf{Dataset} & \textbf{MRR} & \textbf{H10} & \textbf{MRR} & \textbf{H10} & \textbf{MRR} & \textbf{H10} & \textbf{MRR} & \textbf{H10} & \textbf{MRR} & \textbf{H10} & \textbf{MRR} & \textbf{H10}  & \textbf{MRR} & \textbf{H10}\\
\midrule
FB-v1 & $.498$ & $.653$ & $.468$ & $.656$ & $.498$ & $.661$ & $.477$ & $.670$ & $.492$ & $\textbf{.687}$ & $\textbf{.503}$ & $.663$ & $.503$ & $.678$ \\
FB-v2 & $.504$ & $.695$ & $.502$ & $.707$ & $.510$ & $.703$ & $.512$ & $.697$ & $.507$ & $\textbf{.718}$ & $\textbf{.529}$ & $.716$ & $.525$ & $.712$ \\
FB-v3 & $.489$ & $.656$ & $.479$ & $.648$ & $.488$ & $.650$ & $.489$ & $.661$ & $.488$ & $.654$ & $\textbf{.497}$ & $.660$ & $.494$ & $\textbf{.661}$ \\
FB-v4 & $.478$ & $.665$ & $.477$ & $.675$ & $.488$ & $.675$ & $.485$ & $.678$ & $.474$ & $.670$ & $\textbf{.489}$ & $\textbf{.679}$ & $.489$ & $.677$ \\
WN-v1 & $.658$ & $.764$ & $.203$ & $.555$ & $\textbf{.705}$ & $.792$ & $.697$ & $.796$ & $.655$ & $.763$ & $.703$ & $\textbf{.811}$ & $.690$ & $.794$ \\
WN-v2 & $.648$ & $.749$ & $.642$ & $.762$ & $\textbf{.698}$ & $.786$ & $.455$ & $.741$ & $.637$ & $.743$ & $.676$ & $.783$ & $.687$ & $\textbf{.789}$ \\
WN-v3 & $.367$ & $.464$ & $.387$ & $.505$ & $.361$ & $.514$ & $.358$ & $.523$ & $.373$ & $.479$ & $\textbf{.416}$ & $.539$ & $.413$ & $\textbf{.542}$ \\
WN-v4 & $.603$ & $.704$ & $.598$ & $.711$ & $.651$ & $.730$ & $.568$ & $.718$ & $.605$ & $.711$ & $\textbf{.657}$ & $\textbf{.738}$ & $.644$ & $.723$ \\
NL-v1 & $.694$ & $.896$ & $.524$ & $.771$ & $.739$ & $.920$ & $.583$ & $.866$ & $.597$ & $.644$ & $\textbf{.749}$ & $\textbf{.930}$ & $.585$ & $.866$ \\
NL-v2& $.516$ & $.715$ & $.507$ & $.699$ & $.551$ & $.728$ & $.550$ & $.750$ & $.528$ & $.735$ & $.565$ & $.754$ & $\textbf{.572}$ & $\textbf{.763}$ \\
NL-v3 & $.510$ & $.690$ & $.493$ & $.666$ & $.550$ & $.728$ & $.548$ & $.729$ & $.526$ & $.723$ & $\textbf{.562}$ & $.737$ & $.560$ & $\textbf{.750}$ \\
NL-v4 & $.483$ & $.704$ & $.491$ & $.715$ & $.505$ & $.728$ & $.517$ & $.746$ & $.505$ & $.730$ & $.510$ & $.737$ & $\textbf{.521}$ & $\textbf{.756}$ \\
ILPC small & $.296$ & $.445$ & $\textbf{.304}$ & $.450$ & $.299$ & $.450$ & $.295$ & $.454$ & $.303$ & $.447$ & $.301$ & $.450$ & $.304$ & $\textbf{.456}$ \\
ILPC large & $.292$ & $.417$ & $\textbf{.305}$ & $\textbf{.427}$ & $.287$ & $.423$ & $.290$ & $.426$ & $.299$ & $.424$ & $.297$ & $.419$ & $.297$ & $.423$ \\
HM 1k & $.058$ & $.122$ & $.064$ & $.126$ & $.065$ & $.122$ & $\textbf{.079}$ & $\textbf{.147}$ & $.079$ & $.132$ & $.042$ & $.068$ & $.036$ & $.076$ \\
HM 3k & $.055$ & $.112$ & $.048$ & $.095$ & $.046$ & $.079$ & $\textbf{.065}$ & $\textbf{.116}$ & $.056$ & $.101$ & $.039$ & $.063$ & $.034$ & $.078$ \\
HM 5k & $.051$ & $.103$ & $.045$ & $.091$ & $.043$ & $.080$ & $\textbf{.057}$ & $\textbf{.104}$ & $.051$ & $.093$ & $.032$ & $.054$ & $.030$ & $.070$ \\
HM indigo & $.446$ & $.649$ & $.439$ & $.649$ & $.446$ & $.649$ & $\textbf{.449}$ & $\textbf{.654}$ & $.432$ & $.644$ & $.440$ & $.651$ & $.438$ & $.650$ \\
\bottomrule
\end{tabular}
\end{table*}

In tables \ref{tab:model_performance_transductive} - \ref{tab:model_performance_inductive_e} we display the zero shot results of $\rm{Ultra}$ and our models $\Ultra$ on all datasets. Here we show the models that have been pre-trained on a mixture of $3$ graphs (see Table~\ref{tab:pretraining_mixtures}).

Additionally we performed finetuning of our best model on all datasets considered here. The detailed results are displayed in Tables~\ref{tab:inductive_e_finetuned}, \ref{tab:inductive_er_finetuned} and \ref{tab:transductive_finetuned}. For finetuning we employed dataset specific hyperparameters as displayed in table~\ref{tab:hyperparameters-finetuning}. Hyperparameters common to all datasets are in table~\ref{tab:hyperparameter-global}. Due to hardware constraints we used we used a lower batch size compared to \citep{galkin2023towards}, \citep{Zhang2025TRIX:Graphs} and \citep{huang2025expressive} which might reduce the finetuned performance for Datasets trained with partial datasets.

\begin{table}[htbp]
\centering
\caption{Finetuned Inductive $(e,r)$ Performance Comparison}
\label{tab:inductive_er_finetuned}
\begin{tabular}{lcccccc}
\toprule
\multirow{2}{*}{Dataset} & \multicolumn{2}{c}{$\rm{Ultra}$} & \multicolumn{2}{c}{$\rm{Motif}$} & \multicolumn{2}{c}{$\Ultra[\cV_3^+]$} \\
\cmidrule(lr){2-3} \cmidrule(lr){4-5} \cmidrule(lr){6-7}
 & MRR & H@10 & MRR & H@10 & MRR & H@10 \\
\midrule
FB-25 & 0.383 & 0.635 & 0.388 & 0.635 & \textbf{0.391} & \textbf{0.642} \\
FB-50 & 0.334 & 0.538 & \textbf{0.340} & \textbf{0.544} & 0.333 & 0.541 \\
FB-75 & 0.400 & 0.598 & 0.399 & \textbf{0.607} & \textbf{0.403} & 0.604 \\
FB-100 & 0.444 & \textbf{0.643} & 0.439 & 0.642 & \textbf{0.445} & 0.640 \\
WK-25 & \textbf{0.321} & \textbf{0.535} & 0.317 & 0.505 & 0.298 & 0.487 \\
WK-50 & 0.140 & 0.280 & 0.160 & 0.304 & \textbf{0.162} & \textbf{0.314} \\
WK-75 & 0.380 & 0.530 & 0.371 & \textbf{0.535} & \textbf{0.387} & 0.529 \\
WK-100 & 0.168 & 0.286 & 0.173 & 0.284 & \textbf{0.180} & \textbf{0.294} \\
NL-0 & \textbf{0.329} & 0.551 & 0.328 & \textbf{0.556} & 0.305 & 0.490 \\
NL-25 & \textbf{0.407} & \textbf{0.596} & 0.390 & 0.580 & 0.353 & 0.540 \\
NL-50 & \textbf{0.418} & \textbf{0.595} & 0.414 & 0.573 & 0.399 & 0.579 \\
NL-75 & \textbf{0.374} & \textbf{0.570} & 0.360 & 0.548 & 0.360 & 0.563 \\
NL-100 & 0.458 & \textbf{0.684} & 0.464 & 0.682 & \textbf{0.477} & 0.661 \\
Metafam & 0.997 & \textbf{1.000} & \textbf{1.000} & \textbf{1.000} & \textbf{1.000} & \textbf{1.000} \\
FBNELL & \textbf{0.481} & 0.661 & \textbf{0.481} & \textbf{0.664} & 0.445 & 0.626 \\
MT1-tax & 0.330 & 0.459 & 0.416 & 0.522 & \textbf{0.429} & \textbf{0.533} \\
MT1-health & 0.380 & 0.467 & 0.385 & \textbf{0.473} & \textbf{0.386} & 0.462 \\
MT2-org & 0.104 & 0.170 & \textbf{0.106} & 0.170 & 0.104 & \textbf{0.175} \\
MT2-sci & 0.311 & 0.451 & \textbf{0.326} & \textbf{0.520} & 0.320 & 0.427 \\
MT3-art & 0.306 & 0.473 & \textbf{0.315} & 0.469 & \textbf{0.315} & \textbf{0.479} \\
MT3-infra & 0.657 & 0.807 & \textbf{0.683} & \textbf{0.827} & \textbf{0.683} & 0.821 \\
MT4-sci & 0.303 & 0.478 & 0.309 & 0.483 & \textbf{0.311} & \textbf{0.489} \\
MT4-health & 0.704 & 0.785 & 0.703 & 0.787 & \textbf{0.709} & \textbf{0.788} \\
\bottomrule
\end{tabular}
\end{table}

\begin{table}[htbp]
\centering
\caption{Finetuned Inductive $(e)$ Performance Comparison}
\label{tab:inductive_e_finetuned}
\begin{tabular}{lcccccc}
\toprule
\multirow{2}{*}{Dataset} & \multicolumn{2}{c}{$\rm{Ultra}$} & \multicolumn{2}{c}{$\rm{Motif}$} & \multicolumn{2}{c}{$\Ultra[\cV_3^+]$} \\
\cmidrule(lr){2-3} \cmidrule(lr){4-5} \cmidrule(lr){6-7}
 & MRR & H@10 & MRR & H@10 & MRR & H@10 \\
\midrule
FB-v1 & 0.509 & 0.670 & \textbf{0.530} & \textbf{0.702} & 0.510 & 0.669 \\
FB-v2 & 0.524 & 0.710 & \textbf{0.557} & \textbf{0.744} & 0.540 & 0.730 \\
FB-v3 & 0.504 & 0.663 & \textbf{0.519} & \textbf{0.684} & 0.509 & 0.665 \\
FB-v4 & 0.496 & 0.684 & \textbf{0.508} & \textbf{0.695} & 0.497 & 0.683 \\
WN-v1 & 0.685 & 0.793 & 0.703 & \textbf{0.806} & \textbf{0.705} & 0.789 \\
WN-v2 & 0.679 & 0.779 & 0.680 & \textbf{0.781} & \textbf{0.694} & 0.779 \\
WN-v3 & 0.411 & 0.546 & \textbf{0.466} & \textbf{0.590} & 0.433 & 0.550 \\
WN-v4 & 0.614 & 0.720 & 0.659 & 0.733 & \textbf{0.662} & \textbf{0.748} \\
NL-v1 & 0.757 & 0.878 & 0.712 & 0.873 & \textbf{0.811} & \textbf{0.931} \\
NL-v2 & \textbf{0.575} & 0.761 & 0.566 & \textbf{0.765} & 0.572 & 0.761 \\
NL-v3 & 0.563 & 0.755 & 0.580 & 0.764 & \textbf{0.589} & \textbf{0.781} \\
NL-v4 & 0.469 & 0.733 & 0.507 & 0.740 & \textbf{0.537} & \textbf{0.761} \\
ILPC Small & 0.303 & \textbf{0.453} & 0.302 & 0.449 & \textbf{0.307} & 0.450 \\
ILPC Large & \textbf{0.308} & 0.431 & 0.307 & \textbf{0.432} & 0.313 & \textbf{0.435} \\
HM-1k & 0.042 & 0.100 & \textbf{0.067} & \textbf{0.107} & 0.043 & 0.098 \\
HM-3k & 0.030 & 0.090 & \textbf{0.054} & \textbf{0.103} & 0.025 & 0.062 \\
HM-5k & 0.025 & 0.068 & \textbf{0.049} & \textbf{0.091} & 0.025 & 0.050 \\
HM-Indigo & \textbf{0.432} & \textbf{0.639} & 0.426 & 0.635 & 0.393 & 0.520 \\
\bottomrule
\end{tabular}
\end{table}

\begin{table}[htbp]
\centering
\caption{Finetuned Transductive Performance Comparison}
\label{tab:transductive_finetuned}
\begin{tabular}{lcccccc}
\toprule
\multirow{2}{*}{Dataset} & \multicolumn{2}{c}{$\rm{Ultra}$} & \multicolumn{2}{c}{$\rm{Motif}$} & \multicolumn{2}{c}{$\Ultra[\cV_3^+]$} \\
\cmidrule(lr){2-3} \cmidrule(lr){4-5} \cmidrule(lr){6-7}
 & MRR & H@10 & MRR & H@10 & MRR & H@10 \\
\midrule
CoDEx Small & 0.490 & \textbf{0.686} & 0.490 & 0.680 & \textbf{0.496} & 0.684 \\
CoDEx Large & 0.343 & 0.478 & 0.355 & 0.489 & \textbf{0.359} & \textbf{0.495} \\
NELL-995 & 0.509 & 0.660 & 0.514 & 0.655 & \textbf{0.547} & \textbf{0.678} \\
YAGO310 & 0.557 & 0.710 & 0.603 & \textbf{0.735} & \textbf{0.607} & \textbf{0.735} \\
WDsinger & 0.417 & 0.526 & 0.423 & 0.532 & \textbf{0.431} & \textbf{0.535} \\
NELL23k & 0.268 & 0.450 & 0.256 & 0.441 & \textbf{0.270} & \textbf{0.453} \\
FB15k237(10) & 0.254 & 0.411 & 0.254 & 0.411 & \textbf{0.263} & \textbf{0.416} \\
FB15k237(20) & 0.274 & \textbf{0.445} & 0.273 & 0.444 & \textbf{0.281} & \textbf{0.445} \\
FB15k237(50) & 0.325 & 0.528 & 0.323 & 0.523 & \textbf{0.335} & \textbf{0.531} \\
Hetionet & 0.399 & 0.538 & 0.446 & 0.575 & \textbf{0.464} & \textbf{0.599} \\
\bottomrule
\end{tabular}
\end{table}

\begin{table}[ht!]
    \centering
    \caption{Average finetuned link prediction MRR and H\@10 over 51 KGs. Baseline results are taken from~\citep{huang2025expressive}. $\cP_n,$ O and C stand for n-, open, and closed paths; and N-M stands for many-to-many subgraphs} 
    \begin{tabular}{l|lll|llllll|| ll}
    \toprule
     \multirow{3}{*}{Model} & \multicolumn{3}{c|}{\multirow{2}{*}{Structural Vocabulary}} & \multicolumn{2}{c}{Ind.$(e)$
     } & \multicolumn{2}{c}{Ind.$(e,r)$} & \multicolumn{2}{c}{Transd.} & \multicolumn{2}{c}{Total Avg.} \\
     &&&& \multicolumn{2}{c}{ (18 KGs)} & \multicolumn{2}{c}{ (23 KGs)} & \multicolumn{2}{c}{(10 KGs)} & \multicolumn{2}{c}{(51 KGs)}\\
     \cline{2-12}
     &$\cV$& Definition & $\# \cV$ & MRR & H\@10 & MRR & H\@10& MRR & H\@10 & MRR & H\@10\\
     \midrule

$\rmm{Ultra}$ & $\cU_{2}$&$\cP_2$ & $4$ & $.442$ & $.582$ & $.397$ & $.556$ & $.384$ & $.543$ & $.410$ & $.563$\\ 
$\Motif$ & $\cU_{3}$ &$\cP_3$ & $12$ & $\textbf{.455}$ & $\textbf{.594}$ & $\textbf{.401}$ & $\textbf{.558}$ & $.394$ & $.549$ & $.419$ & $\textbf{.569}$ \\
$\Ultra$ & $\cV_{3}$&O.~\&~C.~$\cP_3$  & $24$ & $\textbf{.455}$ & $.581$ & $.400$ & $.551$ & $\textbf{.405}$ & $\textbf{.557}$ & $\textbf{.420}$ & $.563$\\ 
\bottomrule
    \end{tabular}
    \label{tab:finetuned_averageresults}
\end{table}

\begin{table}[t]
    \centering
    \caption{Hyperparameters for fine-tuning $\Ultra$. Full represents a whole epoch with the entire dataset being used}
    \label{tab:hyperparameters-finetuning}
\begin{tabular}{lcc}
\toprule
\textbf{Datasets} & \textbf{Epoch} & \textbf{Batch per Epoch} \\
\midrule
 FB 25-100 & 3 & full   \\
 WK 25-100 & 3 & full   \\
 NL 0-100 & 3 & full  \\
 MT1-MT4 & 3 & full  \\
 Metafam, FBNELL & 3 & full   \\
 \midrule
 FB v1-v4 & 1 & full   \\
 WN v1-v4 & 1 & full   \\
 NL v1-v4 & 3 & full   \\
 ILPC Small & 3 & full   \\
 ILPC Large & 1 & 1000   \\
 HM 1k-5k, Indigo & 1 & 100   \\
 \midrule
 FB15k237 & 1 & full  \\
 WN18RR & 1 & full   \\
 CoDEx Small & 1 & 4000   \\
 CoDEx Medium & 1 & 4000   \\
 CoDEx Large & 1 & 2000  \\
 NELL-995 & 1 & full   \\
 YAGO310 & 1 & 2000   \\
 WDsinger & 3 & full   \\
 NELL23k & 3 & full   \\
 FB15k237(10) & 1 & full   \\
 FB15k237(20) & 1 & full  \\
 FB15k237(50) & 1 & 1000   \\
 Hetionet & 1 & 4000   \\
\bottomrule
\end{tabular}
\end{table}

\begin{table}[t]
    \centering
    \caption{Global hyper-parameters for fine-tuning.}
    \label{tab:hyperparameter-global}
\begin{tabular}{cc}
\toprule
\textbf{Hyperparameter}& \textbf{Value}\\
\midrule
Optimizer & AdamW \\
Learning rate & 0.0005 \\
Adversarial temperature & 1 \\
\# Negatives & 256 \\
Batch size & 8 \\
\# Repetitions & 1\\
\bottomrule
\end{tabular}
\end{table}
	
	\section{Pretraining Scaling}\label{sec:pretraining-scaling}
We investigated the scaling behavior of our approach with different sizes of the relational vocabulary. For each vocabulary set detailed above we conducted pretraining on each of the pretraining mixtures employed in \citep{galkin2023towards} (see Table~\ref{tab:pretraining_mixtures}). Compared to \citep{galkin2023towards} we had to decrease the batch sizes to be able to train all models with the same parameters.

\begin{table*}[!ht]
    \centering
    \caption{Graphs and training parameters in different pre-training mixtures in Figures~\ref{fig:inductive-e-metrics}, \ref{fig:inductive-er-metrics} and \ref{fig:transductive-metrics}}
    \begin{tabular}{lccccccc}\toprule
			&1 &2 &3 &4 &5 &6 \\\midrule
			FB15k237 &\checkmark &\checkmark &\checkmark &\checkmark &\checkmark &\checkmark \\
			WN18RR & &\checkmark &\checkmark &\checkmark &\checkmark &\checkmark \\
			CoDEx-M & & &\checkmark &\checkmark &\checkmark &\checkmark \\
			NELL995 & & & &\checkmark &\checkmark &\checkmark \\
			YAGO 310 & & & & &\checkmark &\checkmark \\
			ConceptNet100k & & & & & &\checkmark  \\
			\midrule
			Batch size &32 &16 &16 &16 &8 &8\\
			\# steps &200,000 &400,000 &300,000 &400,000 &200,000 &200,000 \\
			\bottomrule
		\end{tabular}
    \label{tab:pretraining_mixtures}
\end{table*}

\begin{figure*}
    \centering
    \includegraphics[width=1\linewidth]{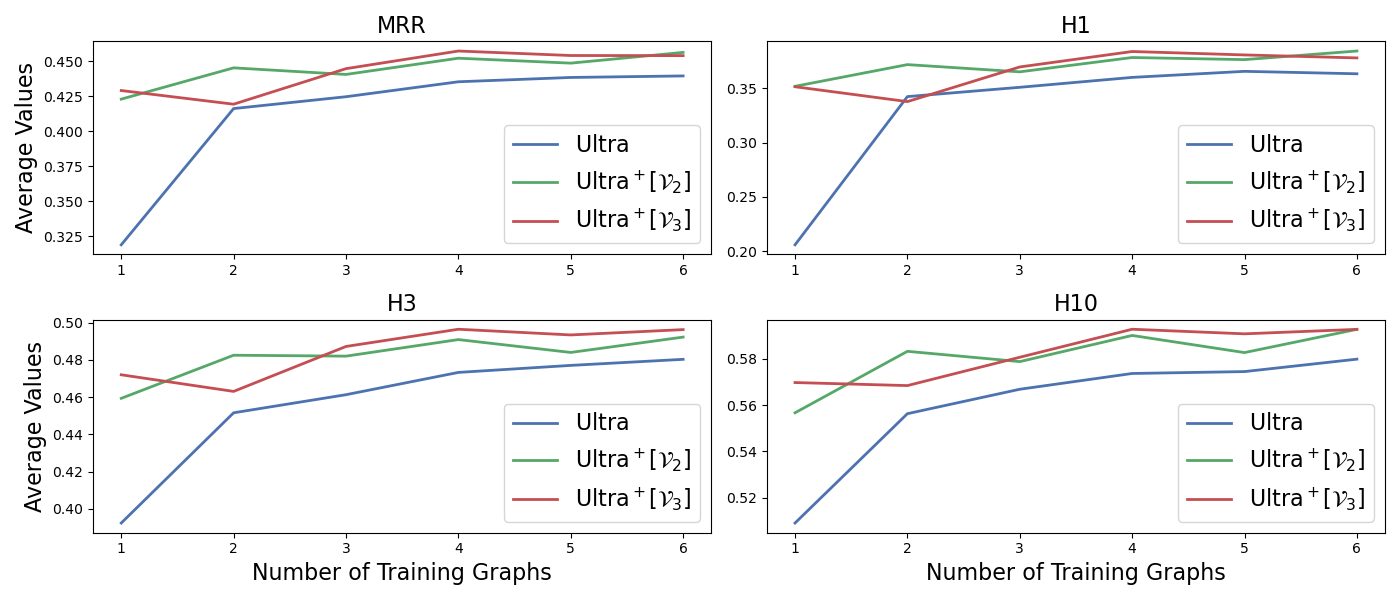}
    \caption{Average performance on 18 inductive (e) datasets of our $\Ultra$ models compared with Ultra, pretrained on 1 - 6 pretraining Graphs (see Table~\ref{tab:pretraining_mixtures}).  }
    \label{fig:inductive-e-metrics}
\end{figure*}
\begin{figure*}
    \centering
    \includegraphics[width=1\linewidth]{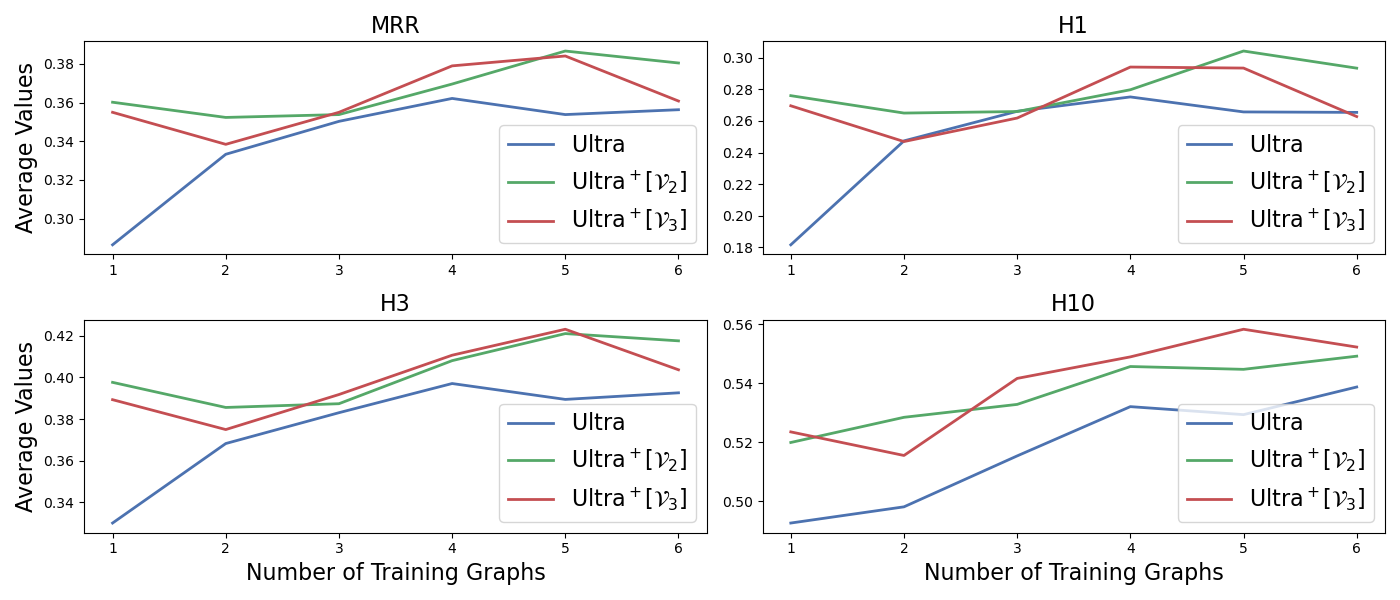}
    \caption{Average performance on 23 inductive (e,r) datasets of our $\Ultra$ models compared with Ultra, pretrained on 1 - 6 pretraining Graphs (see Table~\ref{tab:pretraining_mixtures}).  }
    \label{fig:inductive-er-metrics}
\end{figure*}
\begin{figure*}
    \centering
    \includegraphics[width=1\linewidth]{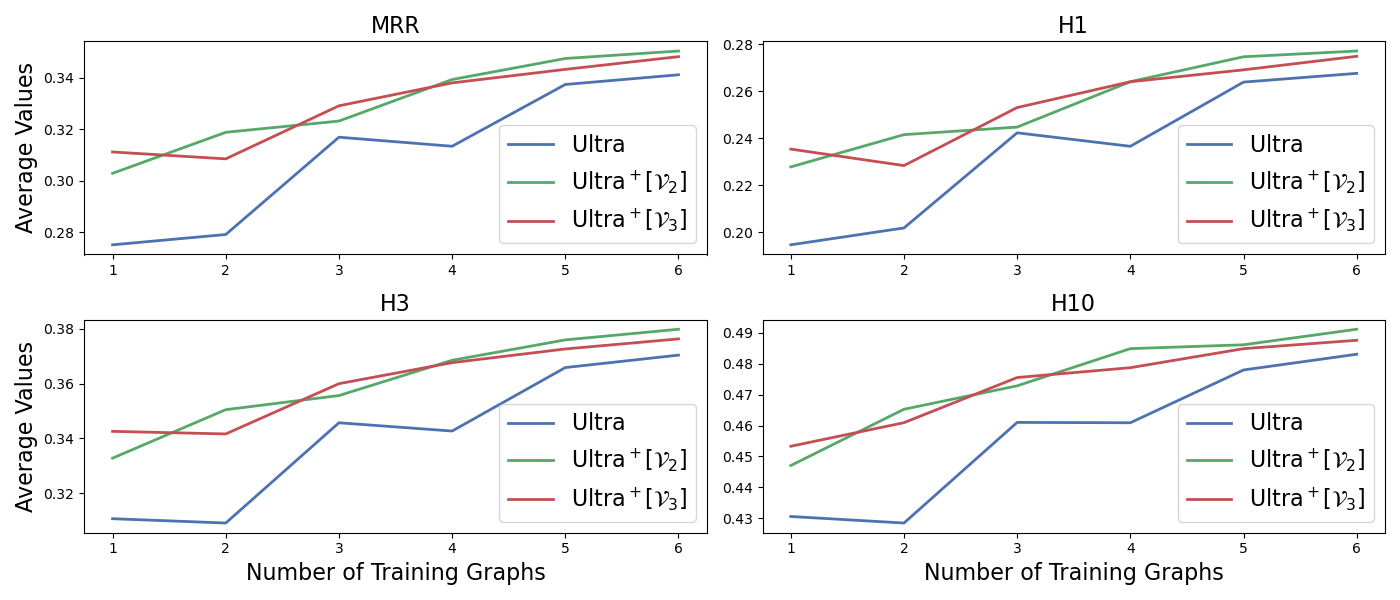}
    \caption{Average performance on 10 transductive datasets of our $\Ultra$ models compared with Ultra, pretrained on 1 - 6 pretraining Graphs (see Table~\ref{tab:pretraining_mixtures}).  }
    \label{fig:transductive-metrics}
\end{figure*}
	
	\section{Details on Relation Graph Computation}
	
	\subsection{Complexity Analysis}

The time complexity of Ultra and Motif are computed in ~\citep{huang2025expressive}. This involves 
\textit{(i)} estimating the  computational complexity of generating the relation graph, by scanning the triples in the KGs and executing the sparse-matrix multiplication, 
\textit{(ii)} in addition to  applying a single forward pass in both the relation and entity encoders. These are  
\begin{equation}
    \label{Eq: complexity Ultra}
    \mathcal{O}\left( \left( \mid \cE \mid^2  \mid \cR \mid + \mid \cE \mid  \mid \cR \mid^2 \right)+L\left( \mid \cR \mid ^2 d +  \mid \cR \mid d^2\right)+L\left( \mid \cT \mid d +  \mid \cE \mid  d^2\right)\right)
\end{equation}
 for {Ultra} and Motif equipped with 2-paths, and 
 \begin{equation}
    \label{Eq: complexity Motif}
    \mathcal{O}\left( \left(\mid \cE \mid \mid \cR \mid^3 + \mid \cE \mid^2  \mid \cR \mid^2 \right)+L\left( \mid \cR \mid ^3 d +  \mid \cR \mid d^2\right)+L\left( \mid \cT \mid d +  \mid \cE \mid  d^2\right)\right)
\end{equation}
for Motif equipped with 3-paths. 
$\Ultra$ replaces step \textit{(i)} by running SPARQL ASK queries on the KGs. This implies we only need to scan the KGs without executing the SPMMs. Since the SPARQL queries in the vocabularies we defined are bounded in size and the filter conditions are simple equalities, their time complexity is $\mathcal{O}(1)$ (refer to Theorem~1 in \citep{perez2006semantics}. Thus, we reduced the time complexity for 2-paths from  $\mathcal{O}\left( \mid \cE \mid^2  \mid \cR \mid + \mid \cE \mid  \mid \cR \mid^2 \right)$ to $\mathcal{O}\left( \mid \cE \mid^2  \mid \cR \mid \right)$
and for 3-paths from  $\mathcal{O}\left( \mid \cE \mid  \mid \cR \mid^3 + \mid \cE \mid^2  \mid \cR \mid^2 \right)$ to $\mathcal{O}\left( \mid \cE \mid  \mid \cR \mid^3\right)$. As $\Ultra$ constructs binary relation graphs, {Ultra} and $\Ultra$ have the same forward pass time complexity for both 2- and 3-paths. In overall, 
\begin{equation}
    \label{Eq: complexity Ultra+2}
    \mathcal{O}\left( \mid \cE \mid^2  \mid \cR \mid +L\left( \mid \cR \mid ^2 d +  \mid \cR \mid d^2\right)+L\left( \mid \cT \mid d +  \mid \cE \mid  d^2\right)\right)
\end{equation}
 for $\Ultra$ equipped with 2-paths, and 
 \begin{equation}
    \label{Eq: complexity Ultra+3}
    \mathcal{O}\left( \mid \cE \mid \mid \cR \mid^3 +L\left( \mid \cR \mid ^2 d +  \mid \cR \mid d^2\right)+L\left( \mid \cT \mid d +  \mid \cE \mid  d^2\right)\right)
\end{equation}
for $\Ultra$ equipped with 3-paths. 




\subsection{Experimental Analysis}
In Table \ref{tab:runtime_comparison_spmm_vs_query_based},  we compare the computation of relation graphs based on sparse matrix multiplication with query-based computation. 
\begin{table*}[!h!]
\centering
\caption{Runtime and memory usage comparison for relation graph computation using  $\cV_2$ with the formulation in \ref{sec:Rel_graph_SPMM} and the Query based computation. Time is displayed in hours:minutes:seconds}
\label{tab:runtime_comparison_spmm_vs_query_based}
\begin{tabular}{lccc | ccc}
\toprule
\multirow{2}{*}{\textbf{Dataset}} & \multicolumn{3}{c}{ \textbf{Query based} } & \multicolumn{3}{c}{ \textbf{SPMM} }\\
\cmidrule{2-7}
&time & RAM usage & VRAM usage & time & RAM & VRAM (GPU) \\
\midrule
WM18RR & 00:00:08 & 10 GB & --  & 00:10:24 & 10 GB & 50 GB\\
FB15k237 & 00:01:03 & 23 GB & -- & 01:52:43 & 40 GB & 50 GB\\
CODEX Medium & 00:00:30 & 12 GB &  -- & 01:07:52 & 10 GB & 50 GB\\

\bottomrule
\end{tabular}
\end{table*}

\textbf{Implementation and Experiment Details} Batching is used in the implementation of the SPMM-based computation of the relation graph, since the intermediate matrix products are too large to fit into memory. Further improvements to our implementation are possible, but not to the extent that the computation will reach the speed of Query based computations.

The Query based relation graph computation was implemented using rdflib \citep{Krech_RDFLib_2025} and the oxrdflib extension based on oxigraph \citep{Pellissier_Tanon_Oxigraph} which provides efficient SPARQL query resolution and supports large datasets.

This comparison has been executed on a Machine with 64cpu (2 * 32 core AMD EPYC) cores 256GB RAM and an Nvidia A100 (80GB) GPU.
	
	\subsection{SPMM formulation of $\Ultra$ Vocabulary} \label{sec:Rel_graph_SPMM}
The adjacency matrices for the relation graphs in \citep{galkin2023towards} and \citep{huang2025expressive} had convenient representaitons in terms of Sparse Matrix Multiplications (SPMM) expressed as products of the adjacency matrix $A \in \R^{n\times n}$ of the KG and two matrices $\E_h\in \R^{n\times m}$ and $E_t \in \R^{m\times n}$. The following formulas are used to construct the binary edges of the relation graph in Motif (see Section F in \citep{huang2025expressive})  and Ultra (see Section B in \citep{galkin2023towards}): 

\begin{align*}
\boldsymbol{A}_{\textit{h2h}} & =\operatorname{spmm}\left(\boldsymbol{E}_h^T, \boldsymbol{E}_h\right) \in \mathbb{R}^{ m \times m} \\
\boldsymbol{A}_{\textit{t2t}} & =\operatorname{spmm}\left(\boldsymbol{E}_t^T, \boldsymbol{E}_t\right) \in \mathbb{R}^{ m \times m} \\
\boldsymbol{A}_{\textit{h2t}} & =\operatorname{spmm}\left(\boldsymbol{E}_h^T, \boldsymbol{E}_t\right) \in \mathbb{R}^{ m \times m} \\
\boldsymbol{A}_{\textit{t2h}} & =\operatorname{spmm}\left(\boldsymbol{E}_t^T, \boldsymbol{E}_h\right) \in \mathbb{R}^{ m \times m} 
\end{align*}

and the 3-ary edges in Motif:

\begin{align*}
\boldsymbol{A}_{\textit{hfh}} & =\operatorname{spmm}\left(\boldsymbol{E}_h^T, \boldsymbol{A}, \boldsymbol{E}_h\right) \in \mathbb{R}^{ m \times m \times m} \\
\boldsymbol{A}_{\textit{tft}} & =\operatorname{spmm}\left(\boldsymbol{E}_t^T, \boldsymbol{A}, \boldsymbol{E}_t\right) \in \mathbb{R}^{ m \times m \times m} \\
\boldsymbol{A}_{\textit{hft}} & =\operatorname{spmm}\left(\boldsymbol{E}_h^T, \boldsymbol{A}, \boldsymbol{E}_t\right) \in \mathbb{R}^{ m \times m \times m} \\
\boldsymbol{A}_{\textit{tfh}} & =\operatorname{spmm}\left(\boldsymbol{E}_t^T, \boldsymbol{A}, \boldsymbol{E}_h\right) \in \mathbb{R}^{ m \times m \times m} .
\end{align*}

The resulting adjacency matrices fail to distinguish between loops and paths that go through the same entity multiple times. 
One way to enable these models to distinguish between closed and open two-paths is to use the full product with masking. For a graphlet $ \rm{g = uv_z}$ in the vocabulary $\cV_2$ where $\rm{u, v\in \set{f,r }}$ and $\rm{z \in \set{0,c}}$ the adjacency matrix is given by $A_{\rm{g}} = (a_{\rm{g}, ij})_{1\leq i,j\leq m} \in \R^{m\times m \times \norm{\cV_2}}$

\begin{align}
    a_{\rm{uv_o},ij} &= \sum_{\substack{l,m,n \\ l\neq m, m\neq n, n\neq l }} \tau(\rm{u},A^i)_{lm} \cdot \tau(\rm{v},A^j)_{mn}, \\
    a_{\rm{uv_c},ij} &= \sum_{\substack{l,m \\ l\neq m }} \tau(\rm{u},A^i)_{lm} \cdot \tau(\rm{v},A^j)_{ml}\;\;  \text{with} \; \rm{u, v \in \set{f,r}}
\end{align}
where $\tau:\set{\rm{r,f}} \times \R^{n\times n} \rightarrow \R^{n\times n}$:
$$\tau(\rm{u}, A):= \begin{cases} A;& \rm{u=f} \\ A^T;& \rm{u=r}. \end{cases} $$

Similar equations hold for longer path.
	
	\subsection{SPARQL queries}\label{sec:sparql_queries}
We display the query patterns for the  Vocabularies employed in our Experiments in Tables~\ref{tab:3r_patterns}-\ref{tab:nm_patterns}
\onecolumn
\begin{longtable}{|l|p{10cm}|}
\caption{3-Relation Pattern SPARQL Queries correspinding to vocabulary $\cV_3$} \label{tab:3r_patterns} \\
\hline
\textbf{Pattern} & \textbf{SPARQL Query} \\
\hline
\endfirsthead

\multicolumn{2}{c}%
{{\bfseries \tablename\ \thetable{} -- continued from previous page}} \\
\hline
\textbf{Pattern} & \textbf{SPARQL Query} \\
\hline
\endhead

\hline \multicolumn{2}{|r|}{{Continued on next page}} \\ \hline
\endfoot

\hline
\endlastfoot

fffo & \begin{lstlisting}
ASK WHERE {
  ?e0 {rel1} ?e1 .
  ?e1 ?rel_0 ?e2 .
  ?e2 {rel2} ?e3 .
  FILTER(?e0 != ?e1 && ?e0 != ?e2 && 
         ?e1 != ?e2 && ?e1 != ?e3 && 
         ?e2 != ?e3 && ?e0 != ?e3)
}
\end{lstlisting} \\
\hline

fffc & \begin{lstlisting}
ASK WHERE {
  ?e0 {rel1} ?e1 .
  ?e1 ?rel_0 ?e2 .
  ?e2 {rel2} ?e0 .
  FILTER(?e0 != ?e1 && ?e1 != ?e2 && 
         ?e0 != ?e2)
}
\end{lstlisting} \\
\hline

ffro & \begin{lstlisting}
ASK WHERE {
  ?e0 {rel1} ?e1 .
  ?e1 ?rel_0 ?e2 .
  ?e3 {rel2} ?e2 .
  FILTER(?e0 != ?e1 && ?e0 != ?e2 && 
         ?e1 != ?e2 && ?e1 != ?e3 && 
         ?e2 != ?e3 && ?e0 != ?e3)
}
\end{lstlisting} \\
\hline

ffrc & \begin{lstlisting}
ASK WHERE {
  ?e0 {rel1} ?e1 .
  ?e1 ?rel_0 ?e2 .
  ?e0 {rel2} ?e2 .
  FILTER(?e0 != ?e1 && ?e1 != ?e2 && 
         ?e0 != ?e2)
}
\end{lstlisting} \\
\hline

frfo & \begin{lstlisting}
ASK WHERE {
  ?e0 {rel1} ?e1 .
  ?e2 ?rel_0 ?e1 .
  ?e2 {rel2} ?e3 .
  FILTER(?e0 != ?e1 && ?e0 != ?e2 && 
         ?e1 != ?e2 && ?e1 != ?e3 && 
         ?e2 != ?e3 && ?e0 != ?e3)
}
\end{lstlisting} \\
\hline

frfc & \begin{lstlisting}
ASK WHERE {
  ?e0 {rel1} ?e1 .
  ?e2 ?rel_0 ?e1 .
  ?e2 {rel2} ?e0 .
  FILTER(?e0 != ?e1 && ?e1 != ?e2 && 
         ?e0 != ?e2)
}
\end{lstlisting} \\
\hline

frro & \begin{lstlisting}
ASK WHERE {
  ?e0 {rel1} ?e1 .
  ?e2 ?rel_0 ?e1 .
  ?e3 {rel2} ?e2 .
  FILTER(?e0 != ?e1 && ?e0 != ?e2 && 
         ?e1 != ?e2 && ?e1 != ?e3 && 
         ?e2 != ?e3 && ?e0 != ?e3)
}
\end{lstlisting} \\
\hline

frrc & \begin{lstlisting}
ASK WHERE {
  ?e0 {rel1} ?e1 .
  ?e2 ?rel_0 ?e1 .
  ?e0 {rel2} ?e2 .
  FILTER(?e0 != ?e1 && ?e1 != ?e2 && 
         ?e0 != ?e2)
}
\end{lstlisting} \\
\hline

rffo & \begin{lstlisting}
ASK WHERE {
  ?e1 {rel1} ?e0 .
  ?e1 ?rel_0 ?e2 .
  ?e2 {rel2} ?e3 .
  FILTER(?e0 != ?e1 && ?e0 != ?e2 && 
         ?e1 != ?e2 && ?e1 != ?e3 && 
         ?e2 != ?e3 && ?e0 != ?e3)
}
\end{lstlisting} \\
\hline

rffc & \begin{lstlisting}
ASK WHERE {
  ?e1 {rel1} ?e0 .
  ?e1 ?rel_0 ?e2 .
  ?e2 {rel2} ?e0 .
  FILTER(?e0 != ?e1 && ?e1 != ?e2 && 
         ?e0 != ?e2)
}
\end{lstlisting} \\
\hline

rfro & \begin{lstlisting}
ASK WHERE {
  ?e1 {rel1} ?e0 .
  ?e1 ?rel_0 ?e2 .
  ?e3 {rel2} ?e2 .
  FILTER(?e0 != ?e1 && ?e0 != ?e2 && 
         ?e1 != ?e2 && ?e1 != ?e3 && 
         ?e2 != ?e3 && ?e0 != ?e3)
}
\end{lstlisting} \\
\hline

rfrc & \begin{lstlisting}
ASK WHERE {
  ?e1 {rel1} ?e0 .
  ?e1 ?rel_0 ?e2 .
  ?e0 {rel2} ?e2 .
  FILTER(?e0 != ?e1 && ?e1 != ?e2 && 
         ?e0 != ?e2)
}
\end{lstlisting} \\
\hline

rrfo & \begin{lstlisting}
ASK WHERE {
  ?e1 {rel1} ?e0 .
  ?e2 ?rel_0 ?e1 .
  ?e2 {rel2} ?e3 .
  FILTER(?e0 != ?e1 && ?e0 != ?e2 && 
         ?e1 != ?e2 && ?e1 != ?e3 && 
         ?e2 != ?e3 && ?e0 != ?e3)
}
\end{lstlisting} \\
\hline

rrfc & \begin{lstlisting}
ASK WHERE {
  ?e1 {rel1} ?e0 .
  ?e2 ?rel_0 ?e1 .
  ?e2 {rel2} ?e0 .
  FILTER(?e0 != ?e1 && ?e1 != ?e2 && 
         ?e0 != ?e2)
}
\end{lstlisting} \\
\hline

rrro & \begin{lstlisting}
ASK WHERE {
  ?e1 {rel1} ?e0 .
  ?e2 ?rel_0 ?e1 .
  ?e3 {rel2} ?e2 .
  FILTER(?e0 != ?e1 && ?e0 != ?e2 && 
         ?e1 != ?e2 && ?e1 != ?e3 && 
         ?e2 != ?e3 && ?e0 != ?e3)
}
\end{lstlisting} \\
\hline

rrrc & \begin{lstlisting}
ASK WHERE {
  ?e1 {rel1} ?e0 .
  ?e2 ?rel_0 ?e1 .
  ?e0 {rel2} ?e2 .
  FILTER(?e0 != ?e1 && ?e1 != ?e2 && 
         ?e0 != ?e2)
}
\end{lstlisting} \\
\end{longtable}

\begin{longtable}{|l|p{10cm}|}
\caption{2-Relation Pattern SPARQL Queries corresponding to vocabulary $\cV_2$} \label{tab:2r_patterns} \\
\hline
\textbf{Pattern} & \textbf{SPARQL Query} \\
\hline
\endfirsthead

\multicolumn{2}{c}%
{{\bfseries \tablename\ \thetable{} -- continued from previous page}} \\
\hline
\textbf{Pattern} & \textbf{SPARQL Query} \\
\hline
\endhead

\hline \multicolumn{2}{|r|}{{Continued on next page}} \\ \hline
\endfoot

\hline
\endlastfoot
ffo & \begin{lstlisting}
ASK WHERE {
  ?e0 {rel1} ?e1 .
  ?e1 {rel2} ?e2 .
  FILTER(?e0 != ?e1 && ?e1 != ?e2 && 
         ?e0 != ?e2)
}
\end{lstlisting} \\
\hline

ffc & \begin{lstlisting}
ASK WHERE {
  ?e0 {rel1} ?e1 .
  ?e1 {rel2} ?e0 .
  FILTER(?e0 != ?e1)
}
\end{lstlisting} \\
\hline

fro & \begin{lstlisting}
ASK WHERE {
  ?e0 {rel1} ?e1 .
  ?e2 {rel2} ?e1 .
  FILTER(?e0 != ?e1 && ?e1 != ?e2 && 
         ?e0 != ?e2)
}
\end{lstlisting} \\
\hline

frc & \begin{lstlisting}
ASK WHERE {
  ?e0 {rel1} ?e1 .
  ?e0 {rel2} ?e1 .
  FILTER(?e0 != ?e1)
}
\end{lstlisting} \\
\hline

rfo & \begin{lstlisting}
ASK WHERE {
  ?e1 {rel1} ?e0 .
  ?e1 {rel2} ?e2 .
  FILTER(?e0 != ?e1 && ?e1 != ?e2 && 
         ?e0 != ?e2)
}
\end{lstlisting} \\
\hline

rfc & \begin{lstlisting}
ASK WHERE {
  ?e1 {rel1} ?e0 .
  ?e1 {rel2} ?e0 .
  FILTER(?e0 != ?e1)
}
\end{lstlisting} \\
\hline

rro & \begin{lstlisting}
ASK WHERE {
  ?e1 {rel1} ?e0 .
  ?e2 {rel2} ?e1 .
  FILTER(?e0 != ?e1 && ?e1 != ?e2 && 
         ?e0 != ?e2)
}
\end{lstlisting} \\
\hline

rrc & \begin{lstlisting}
ASK WHERE {
  ?e1 {rel1} ?e0 .
  ?e0 {rel2} ?e1 .
  FILTER(?e0 != ?e1)
}
\end{lstlisting} \\
\hline

\end{longtable}

\begin{longtable}{|l|p{10cm}|}
\caption{N-M Pattern SPARQL Queries corresponding to vocabulary$\cV_{\bullet}^+$} \label{tab:nm_patterns} \\
\hline
\textbf{Pattern} & \textbf{SPARQL Query} \\
\hline
\endfirsthead

\multicolumn{2}{c}%
{{\bfseries \tablename\ \thetable{} -- continued from previous page}} \\
\hline
\textbf{Pattern} & \textbf{SPARQL Query} \\
\hline
\endhead

\hline \multicolumn{2}{|r|}{{Continued on next page}} \\ \hline
\endfoot

\hline
\endlastfoot

ffo\_1-2 & \begin{lstlisting}
ASK WHERE {
  ?e0 {rel1} ?e1 .
  ?e1 {rel2} ?e2 .
  ?e1 {rel2} ?e3 .
  FILTER(?e0 != ?e1 && ?e1 != ?e2 && 
         ?e2 != ?e3 && ?e3 != ?e0 && 
         ?e0 != ?e2 && ?e1 != ?e2)
}
\end{lstlisting} \\
\hline

fro\_1-2 & \begin{lstlisting}
ASK WHERE {
  ?e0 {rel1} ?e1 .
  ?e2 {rel2} ?e1 .
  ?e3 {rel2} ?e1 .
  FILTER(?e0 != ?e1 && ?e1 != ?e2 && 
         ?e2 != ?e3 && ?e3 != ?e0 && 
         ?e0 != ?e2 && ?e1 != ?e2)
}
\end{lstlisting} \\
\hline

rfo\_1-2 & \begin{lstlisting}
ASK WHERE {
  ?e1 {rel1} ?e0 .
  ?e1 {rel2} ?e2 .
  ?e1 {rel2} ?e3 .
  FILTER(?e0 != ?e1 && ?e1 != ?e2 && 
         ?e2 != ?e3 && ?e3 != ?e0 && 
         ?e0 != ?e2 && ?e1 != ?e2)
}
\end{lstlisting} \\
\hline

rro\_1-2 & \begin{lstlisting}
ASK WHERE {
  ?e1 {rel1} ?e0 .
  ?e2 {rel2} ?e1 .
  ?e3 {rel2} ?e1 .
  FILTER(?e0 != ?e1 && ?e1 != ?e2 && 
         ?e2 != ?e3 && ?e3 != ?e0 && 
         ?e0 != ?e2 && ?e1 != ?e2)
}
\end{lstlisting} \\
\hline

ffo\_2-2 & \begin{lstlisting}
ASK WHERE {
  ?e0 {rel1} ?e2 .
  ?e1 {rel1} ?e2 .
  ?e2 {rel2} ?e3 .
  ?e2 {rel2} ?e4 .
  FILTER(?e0 != ?e1 && ?e1 != ?e2 && 
         ?e2 != ?e3 && ?e3 != ?e0 && 
         ?e0 != ?e2 && ?e1 != ?e2 && 
         ?e4 != ?e0 && ?e4 != ?e1 && 
         ?e4 != ?e2 && ?e4 != ?e3)
}
\end{lstlisting} \\
\hline

fro\_2-2 & \begin{lstlisting}
ASK WHERE {
  ?e0 {rel1} ?e2 .
  ?e1 {rel1} ?e2 .
  ?e3 {rel2} ?e2 .
  ?e4 {rel2} ?e2 .
  FILTER(?e0 != ?e1 && ?e1 != ?e2 && 
         ?e2 != ?e3 && ?e3 != ?e0 && 
         ?e0 != ?e2 && ?e1 != ?e2 && 
         ?e4 != ?e0 && ?e4 != ?e1 && 
         ?e4 != ?e2 && ?e4 != ?e3)
}
\end{lstlisting} \\
\hline

rfo\_2-2 & \begin{lstlisting}
ASK WHERE {
  ?e2 {rel1} ?e0 .
  ?e2 {rel1} ?e1 .
  ?e2 {rel2} ?e3 .
  ?e2 {rel2} ?e4 .
  FILTER(?e0 != ?e1 && ?e1 != ?e2 && 
         ?e2 != ?e3 && ?e3 != ?e0 && 
         ?e0 != ?e2 && ?e1 != ?e2 && 
         ?e4 != ?e0 && ?e4 != ?e1 && 
         ?e4 != ?e2 && ?e4 != ?e3)
}
\end{lstlisting} \\
\hline

rro\_2-2 & \begin{lstlisting}
ASK WHERE {
  ?e2 {rel1} ?e0 .
  ?e2 {rel1} ?e1 .
  ?e3 {rel2} ?e2 .
  ?e4 {rel2} ?e2 .
  FILTER(?e0 != ?e1 && ?e1 != ?e2 && 
         ?e2 != ?e3 && ?e3 != ?e0 && 
         ?e0 != ?e2 && ?e1 != ?e2 && 
         ?e4 != ?e0 && ?e4 != ?e1 && 
         ?e4 != ?e2 && ?e4 != ?e3)
}
\end{lstlisting} \\
\hline

\end{longtable}

\twocolumn
	
\end{document}